\documentclass[letterpaper, 10pt]{article}
\usepackage[utf8]{inputenc}
\usepackage[margin = 1in]{geometry}
\usepackage{smile}
\usepackage{hyperref}
\usepackage{xcolor}
\usepackage[round,compress]{natbib}
\usepackage{parskip}
\usepackage{subcaption}

\usepackage{url}
\hypersetup{
  colorlinks,
  citecolor=blue!70!black,
  linkcolor=blue!70!black,
  urlcolor=blue!70!black
}

\usepackage{notation}

\newcommand{\Saug}{\cS_{\rm aug}}
\newcommand{\tSaug}{\tilde{\cS}_{\rm aug}}
\newcommand{\Raug}{R_{\rm aug}}
\newcommand{\tRaug}{\tilde{R}_{\rm aug}}
\newcommand{\Paug}{P_{\rm aug}}
\newcommand{\tPaug}{\tilde{P}_{\rm aug}}
\newcommand{\paug}[1]{p_{{#1}, \textrm{aug}}}
\newcommand{\tpaug}[1]{\tilde{p}_{{#1}, \textrm{aug}}}
\newcommand{\raug}[1]{r_{{#1}, \textrm{aug}}}
\newcommand{\traug}[1]{\tilde{r}_{{#1}, \textrm{aug}}}
\newcommand{\cPaug}[1]{\cP_{{#1}, \textrm{aug}}}

\newcommand{\Qaug}[1]{Q_{{#1}, \textrm{aug}}}
\newcommand{\tQaug}[1]{\tilde{Q}_{{#1}, \textrm{aug}}}
\newcommand{\Vaug}[1]{V_{{#1}, \textrm{aug}}}
\newcommand{\tVaug}[1]{\tilde{V}_{{#1}, \textrm{aug}}}

\newcommand{\subopt}{\texttt{Regret}}
\newcommand{\gap}{\texttt{gap}}
\newcommand{\mdpaug}{{\tt MDP}_{\rm aug}}
\newcommand{\tmdpaug}{\tilde{\tt MDP}_{\rm aug}}

\newcommand{\pilo}{\Pi_{\rm LO}}
\newcommand{\vnodelay}[1]{V_{{#1}, \rm nodelay}^*}
\newcommand{\vdelay}[1]{V_{{#1}, \rm delay}^*}
\newcommand{\pinodelay}{\pi_{\rm nodelay}^*}
\newcommand{\pidelay}{\pi_{\rm delay}^*}

\newcommand{\as}[2]{\ab_{{#1}:{#2}}}
\newcommand{\fkb}{\mathfrak{b}}

\makeatletter
\newenvironment{protocol}[1][htb]{%
    \renewcommand{\ALG@name}{Protocol}
   \begin{algorithm}[#1]%
  }{\end{algorithm}}
\makeatother

\title{\bf \LARGE Efficient Reinforcement Learning with Impaired Observability: Learning to Act with Delayed and Missing State Observations\thanks{Correspondence to Mengdi Wang. Emails: \texttt{$\{$minshuochen, mengdiw$\}$@princeton.edu}}}

\author{Minshuo Chen$^1$ \quad Jie Meng$^2$ \quad Yu Bai$^3$ \quad Yinyu Ye$^4$ \quad H. Vincent Poor$^1$ \quad Mengdi Wang$^1$ \\
\vspace{0.05in} $^1$Princeton University \quad $^2$Tsinghua University \quad $^3$Salesforce Research \quad $^4$Stanford University}

\begin{document}
\maketitle

\begin{abstract}
In real-world reinforcement learning (RL) systems, various forms of {\it impaired observability} can complicate matters. These situations arise when an agent is unable to observe the most recent state of the system due to latency or lossy channels, yet the agent must still make real-time decisions. This paper introduces a theoretical investigation into efficient RL in control systems where agents must act with delayed and missing state observations. We present algorithms and establish near-optimal regret upper and lower bounds, of the form $\tilde{\mathcal{O}}(\sqrt{{\rm poly}(H) SAK})$, for RL in the delayed and missing observation settings. Here $S$ and $A$ are the sizes of state and action spaces, $H$ is the time horizon and $K$ is the number of episodes. Despite impaired observability posing significant challenges to the policy class and planning, our results demonstrate that learning remains efficient, with the regret bound optimally depending on the state-action size of the original system. Additionally, we provide a characterization of the performance of the optimal policy under impaired observability, comparing it to the optimal value obtained with full observability. Numerical results are provided to support our theory.
\end{abstract}

\section{Introduction}\label{sec:intro}

In Reinforcement Learning (RL), an agent engages with an environment in a sequential manner. In an ideal setting, at each time step, the agent would observe the current state of the environment, select an action to perform, and receive a reward \citep{smallwood1973optimal, bertsekas2012dynamic, sutton2018reinforcement, lattimore2020bandit}. However, real-world engineering systems often introduce impaired observability and latency, where the agent may not have immediate access to the instantaneous state and reward information. In systems with lossy communication channels, certain state observations may even be permanently missing, never reaching the agent. Nonetheless, the agent still needs to make real-time decisions based on the available information.

The presence of impaired observability transforms the system into a complex interactive decision process (Figure~\ref{fig:RL_restricted_obs}), presenting challenges for both learning and planning in RL. With limited knowledge about recent states and rewards, the agent's policy must extract information from the observed history and utilize it to make immediate decisions. This introduces significant complexity to the policy class and poses difficulties for RL. Moreover, the loss of information due to permanently missing observations further hampers the efficiency of RL methods. Although a na\"{i}ve approach would involve augmenting the state and action space to create a fully observable Markov Decision Process (MDP), such a method would lead to exponential regret growth in the state-action size.

\paragraph{Why existing methods do not work}
One may be tempted to cast the problem of impaired observability as a Partially Observed MDP (POMDP) problem. However, this would not solve the problem. In a POMDP, the system does not reveal its instantaneous state to the agent but provides an emission state observation conditioned on the latent state. POMDPs are known to suffer from the curse of history \citep{papadimitriou1987complexity, bertsekas2012dynamic, krishnamurthy2016partially}, unless additional assumptions are imposed. Existing efficient algorithms focus on subclasses of POMDPs with decodable or distinguishable partial observations \citep{jin2020sample, uehara2022provably, zhan2022pac, chen2022partially, liu2022optimistic, zhong2022posterior, chen2023lower}, where the unseen instantaneous state can be inferred from recent observations. Unfortunately, MDPs with impaired observability do not fall into these benign subclasses. The reason behind this is that at each time step, a new observation, if any, is in fact a past state. Viewing it as an emission state of the current one leads to a time reversal posterior distribution depending on the underlying transitions, which suffers from the curse of history and makes the POMDP intractable. The problem becomes even worse if some observations are missing. 

Empirical evidence suggests that efficient RL is possible even with impaired state observability \citep{lizotte2008missing, liu2014impact, agarwal2021blind}. However, theoretical understanding of this problem is very limited. One notable work \citet{walsh2007planning} studied learning with constant-time delayed observations. It identified subclasses of MDPs with nearly deterministic transitions that can be efficiently learned. Beyond this special case, efficient RL with impaired observability in MDPs with fully generality remains largely open. 

Some recent works have studied delayed feedback in MDPs \citep{yang2023reduction, howson2023delayed}. This is a different problem where the agent's policy can still access real-time states but learning uses delayed data. Our problem is fundamentally harder because the agent's policy can only access the lossy and delayed history. See Section~\ref{sec:related} for more discussions.

\paragraph{Our results} In this paper, we provide algorithms and regret analyses for learning the optimal policy in tabular MDPs with impaired observability. Note that this optimal policy is a different one from the optimal policy with full observability. To approach this problem, we construct an augmented MDP formulation where the original state space is expanded to include available observations of past states and an action sequence. However, the expanded state space is much larger than the original one and na\"{i}ve application of known methods would lead to exponentially large regret bounds. In our analyses, we exploit the structure of the augmented transition model to achieve efficient learning and  sharp regret bounds. The main results are summarized as follows.

\noindent $\bullet$ For MDPs with stochastic delays, we prove a $\tilde{O}((H \wedge D)^{5/2}\sqrt{H^3SA K})$ regret bound (Theorem~\ref{thm:delay_regret}) compared to the best feasible policy. Here $S$ and $A$ are the sizes of the original state and action spaces, respectively, $H$ is the horizon, $D$ is the maximal length of delay, and $K$ is the number of episodes. We allow the delay to be stochastic and conditionally independent given the current state and action. Accompanying the regret upper bound, we derive a lower bound in Proposition~\ref{prop:regret_lowerbound}. Moreover, we quantify the performance degradation of the optimal value due to impaired observability, compared to the optimal value for fully observable MDPs (Proposition~\ref{prop:pidelay_gap}). We also showcase in Proposition~\ref{prop:d+1_vs_d} that a short delay does not reduce the optimal value, but slightly longer delay leads to substantial degradation.

\noindent $\bullet$ For MDPs with randomly missing observations, we provide an optimistic RL method that provably achieves $\tilde{\cO}(\sqrt{H^3 S^2 A K})$ regret (Proposition~\ref{prop:s2arate}). We also provide a sharper $\tilde{\cO}(H^4\sqrt{S A K})$ regret in the case when the missing rate is sufficiently small (Theorem~\ref{thm:regret_missing}).

To our best knowledge, these results present a first set of theories for RL with delayed and missing observations. Remarkably, our regret bounds nearly match the minimax-optimal regret of standard MDP in their dependence on $S$ and $A$ (noting that the target optimal policies are different in the two cases). These results imply that RL with impaired observability is provably as efficient as RL with full observability (up to poly factors of $H$).

\begin{figure*}
\centering
\includegraphics[width = 0.7\textwidth]{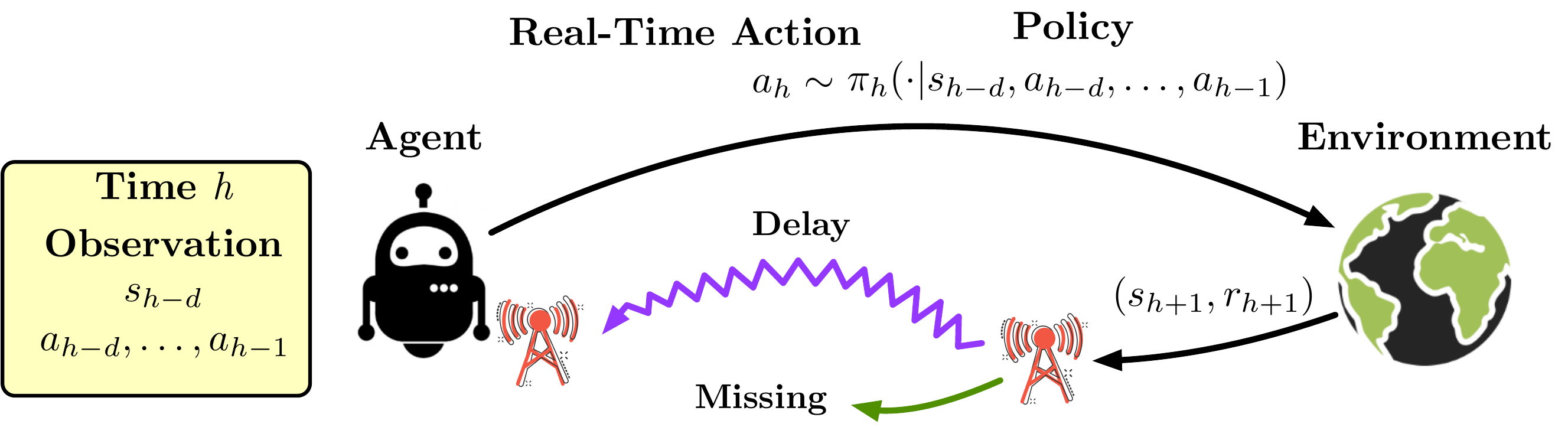}
\caption{Reinforcement learning with impaired observability. At time $h$, the agent observes only the past state $s_{h-d}$ and actions $a_{h-d}, \dots, a_{h-1}$. The policy depends on the observed information.}
\label{fig:RL_restricted_obs}
\end{figure*}

\subsection{Related Work}\label{sec:related}

Efficient algorithms for learning in the standard setting of tabular MDPs without impaired observability have been extensively studied~\citep{kearns2002near, brafman2002r, jaksch2010near, dann2015sample, azar2017minimax, agrawal2017optimistic, jin2018q, dann2019policy, zanette2019tighter, zhang2020almost, domingues2021episodic}, where the minimax optimal regret is $\tO(\sqrt{H^3SAK})$~\citep{azar2017minimax,domingues2021episodic}.

The delayed observation model studied in this paper is related to delayed feedback model in \citet{howson2023delayed, yang2023reduction}, yet the setup is fundamentally different. In delayed feedback, an agent sends a policy to the environment for execution. The environment executes the policy on behalf of the agent for an episode, but the whole trajectory will be returned to the agent after some episodes. The policy executed by the environment is able to ``see" instantaneous states and rewards. It is Markovian and is not played by the agent. Our setting concerns learning executable policies when delayed or missing states appear within an episode. The policy is no longer Markovian and can only prescribe action based on history. Therefore, the algorithms and analyses for delayed feedback MDPs are not applicable to our settings.

Despite the distinct settings, there are existing fruitful results in efficiently learning MDPs or bandits with delayed feedback. Stochastic delayed feedback in bandits is studied in \citet{agarwal2011distributed, dudik2011efficient, joulani2013online, vernade2017stochastic, vernade2020linear, gael2020stochastic, lancewicki2021stochastic}. In a more challenging setting of reinforcement learning, \citet{howson2023delayed} considers tabular MDPs and \citet{yang2023reduction} generalizes to MDPs with function approximation and multi-agent settings.

On the other hand, results analyzing MDPs with missing observations are limited in the literature, although missing data is a commonly recognized issue in applications \citep{garcia2010pattern, jerez2010missing, little2012prevention, emmanuel2021survey}. One notable result is presented in \citet{bouneffouf2020contextual} for bandits with missing rewards.

\subsection{MDP Preliminaries}
An episodic MDP is described by a tuple $(\cS, \cA, H, R, P)$, where $\cS, \cA$ are state and action spaces, respectively, $H$ is the horizon, $R = \{r_h\}_{h=1}^H$ is the reward function and $P = \{p_h\}_{h=1}^H$ is the transition probability. We primarily focus on tabular MDPs, where $S = |\cS|$ and $A = |\cA|$ are both finite. We also assume that the reward is uniformly bounded with $\norm{r_h}_\infty \leq 1$ for any $h$. An agent will interact with the environment for $K$ episodes, hoping to find a good policy to maximize the cumulative reward. Within an episode, at the $h$-th step, the agent chooses an action based on the available information about the environment. After taking the action, the underlying environment produces a reward and transits to the next state. With full state observation, a policy $\pi$ maps the instantaneous state $s$ to an action $a$ or an action distribution. Given such a policy $\pi$, the value function is 
$
V_h^\pi(s_1) = \EE^{\pi} \left[\sum_{h'=h}^H r_h(s_{h'}, a_{h'}) \big| s_h \right],
$ 
where $\EE^{\pi}$ is the policy induced expectation.

\noindent {\bf Notation}: For real numbers $a, b$, $a \wedge b = \min \{a, b\}$ and $\lceil a \rceil$ as the smallest integer larger than $a$. In episodic MDPs, we use a superscript $k$ to denote the index of episodes, and a subscript $h$ to denote the index of time. We denote $\as{i}{j} = \{a_i, \dots, a_j\}$ as the collection of actions from time $i$ to $j$. For two probability distributions $\mu$ and $\nu$, we denote their total variation distance as $\tvnorm{\mu - \nu}$.

\section{Problem formulation}\label{sec:pre}

In this work, we study MDPs with impaired observability. We focus on two practical settings: 1) delayed observations and 2) missing observations.

\subsection{MDPs with Delayed Observations} In any episode, we denote $d_h \in \{0, 1, \dots\}$ as the observational delay of the state and reward at step $h$. That is, we receive $s_h$ and $r_h$ at time $h + d_h$. The delay time $d_h$ can be dependent on the state $s_h$ and action $a_h$ at time $h$. To facilitate analysis, we denote the inter-arrival time between the arrival of observations for step $h$ and $h+1$ as $\Delta_h = d_{h+1} - d_{h}$. With delays, at time $h$, the nearest observable state is denoted as $s_{t_h}$, where $t_h = \argmax~\{I : \sum_{i=0}^I \Delta_i \leq h\}$. Then the executable policy class $$\pilo = \{\pi_h(\cdot | s_{t_h}, \as{t_h}{h-1})~ \text{for}~ h = 1, \dots, H\}$$ chooses actions depending on the nearest visible state and history actions. We impose the following assumption on the interarrival times.
\begin{assumption}\label{assumption:interarrival}
The interarrival time $\Delta_h$ takes value in $\{0, 1, \dots\}$. The distribution $\cD_h(s_h, a_h)$ of $\Delta_h$ can depend on $(s_h, a_h)$, but is conditionally independent of the MDP transitions given $(s_h, a_h)$.
Furthermore, we denote the maximum length of delay as $D = \max_{h = 1, \dots, H} d_h + 1$.
\end{assumption}
Assumption \ref{assumption:interarrival} does not impose any specific distributional assumption on $\Delta_h$, but only requires that the delayed observations arrive in order and that, at each time step, there is at most one new visible state and reward pair ($\Delta_h \geq 0$). A widely studied example of delays in the literature is that the inter-arrival time is geometrically distributed \citep{winsten1959geometric}. Then the observation sequence $\{h + d_h\}$ is a Bernoulli process, which can be thought of as a discretized version of a Poisson process.

Our delayed observation setting is newly proposed and substantially generalizes the Constant Delayed MDPs (CDMDPs) studied in \citet{brooks1972markov, bander1999markov, katsikopoulos2003markov, walsh2007planning}. When $\Delta_h = 0$ being deterministic for all $h \geq 1$ and $k$, our observation delay coincides with CDMDPs. In CDMDPs, a new past observation is guaranteed to arrive at each time step. However, in contrast, our delay model can result in no new observation at some time steps.

Observation delay leads to difficulty in planning, as the agent can only infer the current state and then choose an action. Therefore, the policy is naturally history dependent. We summarize the interaction protocol of the agent with the environment in Protocol~\ref{alg:protocol}.
\begin{protocol}[h]
\caption{Interaction between the agent and the environment with delayed observations}
\label{alg:protocol}
\begin{algorithmic}[1]
\FOR{episode $k = 1, \dots, K$}
\FOR{time $h = 1, \dots, H$}
\STATE The agent observes a pair of new, if any, state and reward $(s_{t_h}^k, a_{t_h}^k)$. By memory, the agent also has access to past actions $\as{t_h}{h-1}^k$.
\STATE The agent plays action $a_h^k$ according to some executable policy $\pi_h^k \in \pilo$.
\STATE The environment transits to next state $s_{h+1}^k \sim p_h(\cdot|s_h^k, a_h^k)$, which is unobservable to the agent. The environment also decides the delay at step $h+1$ as $d_{h+1}^k=d_h^k+\Delta_h^k$ and $t_{h+1}^k$.
\ENDFOR
\STATE The environment sends all unobserved pairs of states and rewards as well as their corresponding delay time to the agent.
\ENDFOR
\end{algorithmic}
\end{protocol}
At the end of each episode, we can collect all delayed observations, however, these observations are not used in planning. In reality, the agent can collect these observations by waiting after time $H$. Protocol~\ref{alg:protocol} is similar to hindsight observability in POMDPs studied in \citet{lee2023learning}. Yet their analysis for POMDPs is not directly transferable to our settings as mentioned in the introduction.

\subsection{MDPs with Missing Observations} In addition to the stochastic delay in observations, we also consider randomly missing observations. In applications, an agent interacts with the environment through some communication channel. The communication channel is often imperfect and thus, observation can be lost during transmission. This type of missing observation is permanent and we describe in the following assumption.
\begin{assumption}\label{assumption:missing}
Any observation pair (state and reward) is independently observable in the communication channel. The observation rate is $\lambda_h$ depending on $h$, but independent of the MDP transitions. Moreover, there exists a constant $\lambda_0$ such that $\lambda_h \geq \lambda_0$ for all $h$. The agent will be informed when an observation is missing.
\end{assumption}

Equivalently, the missing observation rate in Assumption~\ref{assumption:missing} is $1 - \lambda_h$ and assumes the upper bound of $1 - \lambda_0$. We will show in Section~\ref{sec:regret_missing} that this missing observation rate directly influences the learning efficiency.

\section{Construction of Augmented MDPs}\label{sec:alg}

To tackle the limited observability, we expand the original state space and define an augmented MDP. It will serve as the basis for our subsequent theoretical analysis.

\subsection{Augmented MDP with Expected Reward}
In the remainder of this section, we focus on the delayed observation case and defer the missing observation case to Section \ref{sec:regret_missing}. Define $\tau_h = \{s_{t_h}, \as{t_h}{h-1}, \delta_{t_h}\}$ as an augmented state, where $\delta_{t_h} \in [0, \Delta_{t_h}]$ is the number of steps without receiving new observations after observing $(s_{t_h}, r_{t_h})$. Let $\Saug$ denote the augmented state space of all possible $\tau$'s. Then the original MDP with delayed observations can be reformulated into a state-augmented one ${\tt MDP}_{\rm aug} = (\Saug, \cA, H, \Raug, \Paug)$. The reward is defined as
\begin{align*}
\raug{h}(\tau_h, a_h) = \EE\left[r_h(s_{h}, a_{h}) | \tau_h, a_h \right],
\end{align*}
which is the expected reward given the nearest past state $s_{t_h}$ and historical actions $\as{t_h}{h}$. We can define a belief distribution $\mathfrak{b}_h(s | \tau_h) = \PP(s_h = s | \tau_h)$ so that $\raug{h}(\tau_h, a_h) = \EE_{s \sim \mathfrak{b}_h(\cdot | \tau_h)} [r(s, a_h)]$. Belief distributions are widely adopted in partially observed MDPs \citep{ross2007bayes, poupart2008model}. Yet the belief propagation with delayed observations is not Markovian and is rather complicated compared to that in POMDPs. We will frequently use the belief distribution to study the expressivity of $\pilo$ in Section \ref{sec:expressivity}.

The transition probabilities $\Paug$ are sparse. For any $\tau_h = \{s_{t_h}, \as{t_h}{h-1}, \delta_{t_h}\}$ and $\tau_{h+1} = \{s_{t_{h+1}}, \as{t_{h+1}}{h}, \delta_{t_{h+1}}\}$, we have
\begin{align*}
\paug{h}(\tau_{h+1} | \tau_{h}, a_{h}) =
\begin{cases}
{\tt M}_a(\tau_h, \tau_{h+1})\theta_{\textrm{delay}, t_h} (s_{t_h}, a_{t_h}, \delta_{t_h}) p_{t_h}(s_{t_{h+1}} | s_{t_h}, a_{t_h}), & \text{if}~\delta_{t_{h+1}} = 0 ~\text{and}~ t_{h+1} = t_h + 1 \\
{\tt M}_a(\tau_h, \tau_{h+1}) (1 - \theta_{\textrm{delay}, t_h}(s_{t_h}, a_{t_h}, \delta_{t_h})) & \text{if}~ \delta_{t_{h+1}} = \delta_{t_h} + 1 ~\text{and}~ t_{h+1} = t_h \\
0 & \text{otherwise}
\end{cases},
\end{align*}
where ${\tt M}_a(\tau_h, \tau_{h+1})$ indicates whether or not the rolling actions are matched, i.e.,
\begin{align*}
{\tt M}_a(\tau_h, \tau_{h+1}) = \mathds{1}\{\as{t_h}{h-1} =  \as{t_{h+1}}{h-1}\},
\end{align*}
and $\theta_{\textrm{delay}, t_h}(s_{t_h}, a_{t_h}, \delta_{t_h})$ is defined as
\begin{align*}
\theta_{\textrm{delay}, t_h}(s_{t_h}, a_{t_h}, \delta_{t_h}) = \PP_{t_h}(\Delta_{t_h} = \delta_{t_h} | s_{t_h}, a_{t_h}, \delta_{t_h}) = \frac{\PP_{t_h}(\Delta_{t_h} = \delta_{t_h} | s_{t_h}, a_{t_h})}{1 - \sum_{\delta < \delta_{t_h}} \PP_{t_h}(\Delta_{t_h} = \delta | s_{t_h}, a_{t_h})}.
\end{align*}
The factored form of $\theta_{\textrm{delay}, t_h} (s_{t_h}, a_{t_h}, \delta_{t_h}) p_{t_h}(s_{t_{h+1}} | s_{t_h}, a_{t_h})$ follows from the conditional independence in Assumption~\ref{assumption:interarrival}. We define $Q$-functions and value functions as follows. For any $\tau_h, a_h$ and policy $\pi \in \pilo$, we have
\begin{align*}
\Qaug{h}^\pi(\tau_h, a_h) & = \EE^\pi\left[\sum_{h'=h}^H \raug{h}(\tau_{h'}, a_{h'}) \Big| \tau_h, a_h\right] \quad \text{and} \\
\Vaug{h}^\pi(\tau_h) & = \inner{\Qaug{h}^\pi(\tau_h, \cdot)}{\pi_h(\cdot | \tau_h)}.
\end{align*}
We note that $V_h^\pi$ is equivalent to $\Vaug{h}^{\pi}$ for the same executable policy $\pi \in \pilo$. We also denote $\cPaug{h}$ as the transition operator corresponding to $\Paug$. It can be checked that
\begin{align*}
\Qaug{h}^\pi(\tau_h, a_h) = \raug{h}(\tau_h, a_h) + [\cPaug{h} \Vaug{h}^\pi](\tau_h, a_h).
\end{align*}
$\mdpaug$ also makes all the policies in $\pilo$ executable and Markovian. Meanwhile, the reward function keeps track of the expected reward for all $h \leq H$. Although the expanded state space $\Saug$ is much more complicated than the original state space $\cS$, the sparseness of the transition probabilities still allow for efficient exploration. We note that $\paug{h}$ depends only on the delay distribution and one-step Markov transitions. However, there is still one caveat for learning in $\mdpaug$ -- the reward function depends on the belief distributions, which involve multi-step transitions.

\subsection{Augmented MDP with Past Reward}

To tackle the aforementioned challenge, we further define $\tmdpaug = (\tSaug, \cA, \tilde{H}, \tRaug, \tPaug)$ which shares the optimal policy in $\mdpaug$ with an enlonged horizon $\tilde{H} = 2H$. The state space $\tSaug$ consists of any $\tau_h = \{s_{t_h}, \as{t_h}{h-1 \wedge H}, \delta_{t_h}\}$. Comparing to $\Saug$, we cut off the action at horizon $H$, since $a_h$ for $h > H$ has no influence on the state and reward during $[0, H]$. The reward function is defined as
\begin{align*}
\traug{h}(\tau_h, a_h) = r_{t_h}(s_{t_h}, a_{t_h}) \mathds{1}\{\delta_{t_h} = 0\} \mathds{1}\{t_h \in \{1, \dots, H\}\}.
\end{align*}
By definition, $\tilde{r}_{\rm aug}(\tau_h, a_h)$ is a past reward. More importantly, $\traug{h}(\tau_h, a_h)$ zeros out rewards outside the original horizon $H$. Meanwhile, between the arrival of two consecutive state observations, the reward only counts once. Lastly, the transition probabilities are
\begin{align*}
\tpaug{h}(\tau_{h+1} | \tau_{h}, a_{h}) =
\begin{cases}
{\tt M}_a(\tau_h, \tau_{h+1})\theta_{\textrm{delay}, t_h} (s_{t_h}, a_{t_h}, \delta_{t_h}) p_{t_h}(s_{t_{h+1}} | s_{t_h}, a_{t_h}), & \text{if}~\delta_{t_{h+1}} = 0, t_{h+1} = t_h + 1 ~\text{and}~ h < H \\
{\tt M}_a(\tau_h, \tau_{h+1}) (1 - \theta_{\textrm{delay}, t_h}(s_{t_h}, a_{t_h}, \delta_{t_h})) & \text{if}~ \delta_{t_{h+1}} = \delta_{t_h} + 1, t_{h+1} = t_h ~\text{and}~ h < H \\
{\tt M}_a(\tau_h, \tau_{h+1}) p_{t_h}(s_{t_h + 1} | s_{t_h}, a_{t_h}) & \text{if}~\delta_{t_h+1} = 0, t_{h+1} = t_h + 1 ~\text{and}~ h > H \\
0 & \text{otherwise}
\end{cases}.
\end{align*}

We interpret the transitions as follows. When $h \leq H$, the transition is the same as $\mdpaug$. When $h > H$, we simply wait for unobserved states and rewards to come. As mentioned, actions taken beyond time $H$ are irrelevant. The following proposition asserts an equivalence between the value functions in $\mdpaug$ and $\tmdpaug$.
\begin{proposition}\label{prop:mdp=mdptilde}
Let $\mdpaug$ and $\tmdpaug$ be defined as in the previous paragraphs. Then for any initial state $\tau_1$ and any policy $\pi = \{\pi_h\}_{h=1}^{H} \in \pilo$, it holds that
\begin{align*}
\EE^\pi \left[ \sum_{h=1}^{H} \raug{h}(\tau_h, a_h) \Big| \tau_1 \right] = \EE^\pi \left[\sum_{h=1}^{\tilde{H}} \traug{h}(\tau_h, a_h) \Big| \tau_1 \right],
\end{align*}
where in the right-hand side, the policy for steps $H+1$ to $\tilde{H}$ is arbitrary.
\end{proposition}
The proof is provided in Appendix~\ref{pf:mdp=mdptilde}. Proposition~\ref{prop:mdp=mdptilde} implies that learning in $\mdpaug$ until time $H$ is equivalent to that in $\tmdpaug$ for $\tilde{H}$ steps.

\section{RL with Delayed Observations and Regret Analysis}\label{sec:regret_delay}

In this section, we provide a regret analysis of learning in MDPs with stochastic delays. For the sake of simplicity, we assume the reward is known, however, extension to unknown reward causes no real difficulty. Motivated by the augmented MDP reformulation, we introduce our learning algorithm in Algorithm~\ref{alg:future}. In Line~\ref{alg:data}, unobserved states and rewards are returned to the agent as described in Protocol~\ref{alg:protocol}. Using the data set, we construct bonus functions compensating for the uncertainty in {\it one-step} transitions of the original MDP. This largely sharpens the confidence region, yet still ensures a valid optimism. We emphasize that in Line~\ref{alg:aug_planning}, we are planning on $\tmdpaug$ involving the augmented transitions and expanded states of $\tau \in \tSaug$. Only in this way, can we obtain an executable policy in delayed MDPs.
\begin{algorithm}[h]
\caption{Policy learning for delayed MDPs using $\tmdpaug$}
\begin{algorithmic}[1]\label{alg:future}
\STATE {\bf Input}: Original horizon $H$, extended horizon $\tilde{H}$, policy class $\pilo$, failure probability $\gamma$.
\STATE {\bf Init}: $V_{\tilde{H}+1}(\tau) = 0$ and $Q_{\tilde{H}}(\tau, a) = H$ for any $\tau$ and $a$, data set $\cD^0 = \emptyset$, initial policy $\pi^0$.
\FOR{episode $k = 1, \dots, K$}
\STATE Execute policy $\pi^{k-1}$ for $\tilde{H}$ steps.
\STATE After the episode ends, collect data $\cD^k = \cD^{k-1} \cup \{(s_h^k, a_h^k, r_h^k, \Delta_{h}^k)\}_{h=1}^H$. \label{alg:data}
\STATE On data set $\cD^k$, compute counting numbers $N_h^k(s_h, a_h)$, $N_h^k(s_h, a_h, s_{h+1})$ and $N_h^k(s_h, a_h, \delta_h) = \sum_{j=1}^k \mathds{1}\{s_h^j = s_h, a_h^j = a_h, \Delta_h^j = \delta_h\}$.
\STATE Estimate transition probabilities and delay distributions via
\begin{align*}
\hat{p}_h^k(s_{h+1} | s_h, a_h) = \frac{N_h^k(s_h, a_h, s_{h+1})}{N_h^k(s_{h}, a_h)}, \quad \text{and} \quad \hat{\theta}_{\textrm{delay}, h}^k (s_h, a_h, \delta_h) = \frac{N_h^k(s_h, a_h, \delta_h)}{\sum_{\delta \geq \delta_h} N_h^k(s_h, a_h, \delta)}.
\end{align*}
Then estimate $\tpaug{h}$ in $\tmdpaug$ using $\hat{p}_h^k$ and $\hat{\theta}_{\rm delay}^k$.
\STATE Set the bonus function as
\begin{align*}
b_h^k(\tau_h, a_h)  = c H \left(\sqrt{\frac{(H \wedge D) \iota}{N_{t_h}^k(s_{t_h}, a_{t_h}, \delta_{t_h})}} + \sqrt{\frac{(H \wedge D) \iota}{N_{t_h}^k(s_{t_h}, a_{t_h})}}\right)
\end{align*}
for $\iota = \log \frac{SAKH}{\gamma}$ and $c$ sufficiently large.
\STATE Run optimistic value iteration in $\tmdpaug$ for $\tilde{H}$ steps and obtain $\pi^k \in \pilo$.\label{alg:aug_planning}
\ENDFOR
\STATE {\bf Return}: Learned policy $\pi^k_{1:H}$ for $k = 1, \dots, K$.
\end{algorithmic}
\end{algorithm}

\subsection{Regret Bound}
We define the regret in delayed MDP as
\begin{align*}
\textstyle \subopt(K) = \sum_{k=1}^K \max_{\pi \in \pilo} V_{1}^\pi (s_1^k) - \sum_{k=1}^K V_1^{\pi_k}(s_1^k),
\end{align*}
where $V_1^\pi$ is the value function of the original MDP. Although the regret here is defined on the original MDP, it is equivalent to the regret of the same policy on $\mdpaug$ and further $\tmdpaug$ by Proposition~\ref{prop:mdp=mdptilde}. Note that we are comparing with the best executable policy. The performance degradation caused by observation delay is discussed in Section~\ref{sec:expressivity}. The following theorem bounds the regret.
\begin{theorem}[Regret Bound for Delayed MDPs]\label{thm:delay_regret}
Suppose Assumption~\ref{assumption:interarrival} holds. Let $\gamma \in (0, 1)$ be any failure probability. With probaiblity $1 - \gamma$, the regret of Algorithm~\ref{alg:future} satisfies
\begin{align*}
\subopt(K) \leq c \left((H \wedge D)^{5/2} \sqrt{H^3 SAK \iota} + H^4 S^2 A \iota^2 \right),
\end{align*}
where $\iota = \log \frac{SAHK}{\gamma}$ and $c$ is a constant.
\end{theorem}
The proof is provided in Appendix~\ref{pf:delay_regret}. We now discuss several implications.

\paragraph{Sharp dependence on $S$ and $A$} Theorem~\ref{thm:delay_regret} has a sharp dependence on $S$ and $A$, although the expanded state space $\tSaug$ has a cardinality bounded by $SA^H$. Na\"{i}vely learning and planning in $\tmdpaug$ would suffer from the exponential enlargement of $A^H$. However, we identify the sparse structures in the transition probabilities. As can be seen, $\tpaug{h}$ involves only one-step transitions in the original MDP and some conditionally independent delay distributions. Such structures lead to a rather easy estimation of $\tpaug{h}$, which can be constructed from the estimators of one-step transitions in the original MDP. Meanwhile, the sparse structures make exploration in $\tmdpaug$ efficient.

\paragraph{Influence of the delay distribution} Theorem~\ref{thm:delay_regret} holds for arbitrary conditionally independent delay distributions, even including heavy-tailed distributions. In the worst case of unbounded delays, Theorem~\ref{thm:delay_regret} gives rise to a $\cO(H^4 \sqrt{SAK\iota})$ regret. The reason for this is that if the delay is larger than $H$, then the corresponding state will only be observed after an episode ends and will not be used in planning. Therefore, we can truncate the delay at $H$, regardless of its tail distributions.

When the maximal length of delay $D \leq H$, e.g., CDMDPs with $d_h = D \leq H$ for any $h$, Theorem~\ref{thm:delay_regret} implies that the regret is bounded by 
$$\subopt(K) \leq c\left(D^{5/2} \sqrt{H^3SAK\iota} + H^4 S^2 A\iota^2 \right).$$
We observe that as the length of delay increases, the regret bound enlarges, reflecting the increased difficulty of long delays. Moreover, when $D = 1$ corresponding to no delays, we recover the standard regret bound in tabular MDPs \cite[Theorem 1]{azar2017minimax}. We remark that Theorem~\ref{thm:delay_regret} can be directly extended to $D$ being a high probability upper bound on the length of delay. Yet, the analysis causes no real difficulty and is omitted.

\paragraph{Extension to discounted MDPs} Our algorithm and regret analysis are applicable to episodic infinite-horizon discounted MDPs. Since the reward is uniformly bounded by $1$, we can truncate the horizon so that the tail accumulative reward contributes negligibly. In this way, we convert the infinite-horizon MDP to a finite-horizon MDP. Then we have the following corollary.
\begin{corollary}\label{cor:discounted_regret}
Consider episodic infinite-horizon MDPs with a discount factor $u \in (0, 1)$. Suppose Assumption~\ref{assumption:interarrival} holds. Let $\gamma \in (0, 1)$ be any failure probability. With probability $1 - \gamma$, running Algorithm~\ref{alg:future} with $\tilde{H} = \frac{2\iota}{1- u}$ gives rise to
\begin{align*}
\subopt(K) \leq c \left(\left(\frac{\iota}{1-u} \wedge D\right)^{5/2} \sqrt{\left(\frac{1}{1-u}\right)^3SAK \iota^4} + \left(\frac{1}{1-u}\right)^4 S^2 A \iota^6 \right),
\end{align*}
where $\iota = \log \frac{SAHK}{\gamma}$ and $c$ is a constant.
\end{corollary}
The proof is provided in Appendix~\ref{pf:discounted_regret}. Roughly speaking, Corollary~\ref{cor:discounted_regret} is obtained by taking an ``effective'' horizon of $\frac{\iota}{1-u}$ so that the horizon-truncated value function is almost that for the infinite horizon, where the difference is of the same order as the regret.

\subsection{Lower Bounds on Regret and Planning Complexity}
We first show an accompanying lower bound on the regret and demonstrate the influence of the length of delay.
\begin{proposition}\label{prop:regret_lowerbound}
Let $D$ be an integer and consider CDMDPs with length of delay $D$. Suppose $S \geq 6$, $A \geq 2$, $H \geq 2D$, and $K \geq HSA$. Then for any learning algorithm, there exists a hard MDP instance such that
\begin{align*}
\EE[\subopt(K)] \geq \frac{1}{72\sqrt{2}} H\sqrt{DSAK},
\end{align*}
where the expectation is taken over the randomness in the algorithm and the MDP instance.
\end{proposition}
The proof is provided in Appendix~\ref{pf:regret_lowerbound}. The lower bound in Proposition~\ref{prop:regret_lowerbound} indicates increased complexity when the length of delay increases. Meanwhile, the dependence on $S, A$ and $K$ matches that in the upper bound (Theorem~\ref{thm:delay_regret}). It is noted that there is still a gap between the upper and lower bound in terms of the dependence on $D$ and $H$. We suspect that the lower bound can be further improved; nonetheless, detailed investigation of this issue is left as a future research direction.

We next shift our attention to the planning complexity with delayed observations. It is noted that in the worst case, the planning complexity in Algorithm~\ref{alg:future} grows exponentially with respect to the length of delay. Unfortunately, this is inevitable even for CDMDPs.
\begin{proposition}[Theorem 2 in \cite{walsh2007planning}]\label{prop:planning_lowerbound}
The general CDMDP planning problem is NP-Hard.
\end{proposition}
Yet, there are special subclasses of MDPs still allow a polynomial-time planning, such as nearly deterministic MDPs identified in \cite{walsh2007planning}. While investigating polynomial planning algorithms for specific problem is beyond the scope of this paper, we believe that our augmented MDP formation is compatible with practical planning oracle for accelerated solution of RL.

\subsection{Performance Degradation Caused by Delays}\label{sec:expressivity}
This section is devoted to quantifying the performance degradation caused by delayed observations. In particular, we bound the value difference between the best executable policy and the best Markov policy in a no delay environment. Recall that $V_1$ is the value function of the original MDP. We denote
\begin{align*}
\textstyle  \pinodelay = \argmax_{\pi} V_1^{\pi}(s_1) \quad \text{and} \quad \pidelay = \argmax_{\pi \in \pilo} V_1^{\pi}(s_1)
\end{align*}
as the best vanilla optimal policy and executable policy, respectively.
The values achieved by $\pinodelay$ and $\pidelay$ are denoted as $\vnodelay{1}(s_1)$ and $\vdelay{1}(s_1)$, respectively. The gap between $V^*_{1, \textrm{nodelay}}$ and $V^*_{1, \textrm{delay}}$ quantifies the performance degradation, which is denoted as $$\gap(s_1) = V^*_{1, \textrm{nodelay}}(s_1) - V^*_{1, \textrm{delay}}(s_1).$$
We bound $\gap$ in Proposition~\ref{prop:pidelay_gap}.
\begin{proposition}\label{prop:pidelay_gap}
In the setup of Section~\ref{sec:expressivity}, we have
\begin{align*}
\gap(s_1) & \leq \sum_{h=1}^H \Bigg[\underbrace{\int_{\tau} \left(\EE_{s \sim \fkb_h(\cdot | \tau)}[\max_a r_h(s, a)] - \max_a \EE_{s \sim \fkb_h(\cdot | \tau)}[r_h(s, a)]\right)\left(\rho_{h}^{\pidelay} \wedge \rho_h^{\pinodelay} \right)(\tau) \diff \tau}_{\cE_1} \\
& \quad + 2 \underbrace{\tvnorm{\rho_{h}^{\pinodelay} - \rho_{h}^{\pidelay}}}_{\cE_2}\Bigg].
\end{align*}
where $\rho_{h}^{\pinodelay}$ and $\rho_{h}^{\pidelay}$ are visitation measures induced by $\pinodelay$ and $\pidelay$, respectively.
\end{proposition}
The proof is provided in Appendix~\ref{pf:pidelay_gap}. The term $\cE_1$ is nonnegative due to the convexity of the max operator. The term $\cE_2$ accounts for the difference in the visitation measure. When the original MDP has deterministic transitions, we can check that $\cE_1$ is zero, since the expectation over $s$ is concentrated on a singleton that can be inferred from history. Hence, the visitation measures are also identical, which implies $\vnodelay{1}(s_1) - \vdelay{1}(s_1) = 0$. On the contrary, when $\fkb_h(\cdot | \tau)$ is evenly spread, meaning that the entropy of $\fkb_h$ is high, we potentially suffer from a large performance drop, in that inferring the current state is difficult.

\subsection{(Mysterious) Effect of Delay on the Optimal Value}
To further understand the effect of the delay on the optimal value, we provide the following dichotomy. On the one hand, we show that there exists an MDP instance, such that a constant delay of $d$ steps does not hurt the performance. On the other hand, in the same MDP instance, a constant delay of $d+1$ steps results in a constant performance drop. 

\begin{proposition}\label{prop:d+1_vs_d}
Consider constant-delayed MDPs. Fix a positive integer $d < H$. Then there exists an MDP instance such that the following two items hold simultaneously.

\noindent $\bullet$ When the delay is $d$, it holds that $\frac{1}{K}\sum_{k=1}^K \gap(s_1^k) = 0$.

\noindent $\bullet$ When the delay is $d+1$, it holds that $\frac{1}{K}\sum_{k=1}^K \gap(s_1^k) \geq \frac{1}{2} - \sqrt{\frac{1}{2K} \log \frac{1}{\gamma}}$, with probability $1 - \gamma$.
\end{proposition}
The proof is provided in Appendix~\ref{pf:d+1_vs_d}. We remark that Proposition~\ref{prop:d+1_vs_d} says that observation delay can be dangerous, even with the slightest possible number of steps. The idea behind Proposition~\ref{prop:d+1_vs_d} is consistent with the analysis on $\gap$. In particular, we construct an MDP instance demonstrated in Figure~\ref{fig:d+1_vs_d}, where the reward vanishes at all times but $d+1$. When the delay is $d$, the initial state $s_1$ is revealed and the policy can choose the best action to receive a reward. When the delay is $d+1$, however, there is always a $1/2$ probability of missing the best action for any policy, which leads to a constant performance degradation.
\begin{figure*}[h]
\centering
\includegraphics[width = 0.6\textwidth]{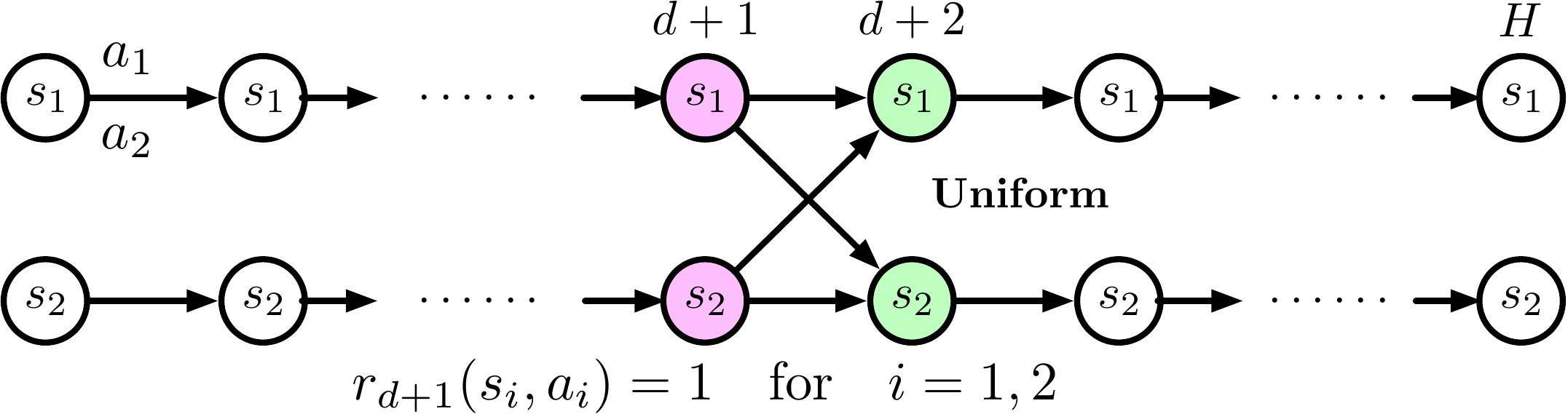}
\caption{MDP instance on two states with two actions. The transition is lazy until time $d$. Then the transition is uniform regardless of actions for time $d+1$. The reward is nonzero only at time $d+1$. This is an example where a delay of length $d$ causes no degradation and a delay of $d+1$ causes a constant performance degradation.}
\label{fig:d+1_vs_d}
\end{figure*}

\section{RL with Missing Observations and Regret Analysis}\label{sec:regret_missing}
We now switch our study to MDPs with missing observations. In such an environment, executable policies share the same structures as delayed MDPs, where an action is taken based on available history information. Compared to delayed observations, learning with missing observations is more challenging. Since unobserved states and rewards are never recovered, we suffer from information loss. Moreover, we will frequently deal with multi-step transitions, due to missing observations between two consecutive visible states.

\subsection{Optimistic Planning with Missing Observations} Despite the difficulty, we present here algorithms that are efficient in learning and planning for MDPs with missing observations. We begin with an optimistic planning algorithm in Algorithm~\ref{alg:opt_planning}. To unify the notation, we denote $s_h^k = \emptyset$ and $r_h^k = \emptyset$ to indicate missing the corresponding observation.

\begin{algorithm}[h]
\caption{Optimistic planning for MDPs with missing observations} 
\begin{algorithmic}[1]\label{alg:opt_planning}
\STATE {\bf Input}: Horizon $H$, observable rate $\lambda_h$.
\STATE {\bf Init}: $\cB^0=\Theta$ to be all possible tabular MDPs, data set $\cD^0 = \emptyset$.
\FOR{episode $k=1,\dots,K$}
\STATE Set policy $\pi^k = \argmax_{\pi\in\pilo} \max_{\theta\in\cB^k} V^\pi_{1, \theta}(s_1^k)$. \label{alg:greedy}
\STATE Play policy $\pi^k$ and collect data $\cD^{k-1} \cup \{(s_h^k,a_h^k,r_h^k)\}_{h=1}^H$.
\STATE Compute counting numbers $N_h^k(s, a) = \sum_{j=1}^k \mathds{1}\{s_h^j = s, a_h^j = a, s_{h+1}^j \neq \emptyset \}$.\label{alg:missing_count}
\STATE Update the confidence set
\begin{align*}
\cB^{k} = \Big\{\theta: \tvnorm{\hat{p}^k_h(\cdot|s,a) -  p^\theta_h(\cdot|s,a)} \le c\sqrt{\frac{S\iota}{N_h^k(s,a)}}~\textrm{for all}~(h,s,a)\Big\} \cap \cB^{k-1},
\end{align*}
where $\hat{p}^k_h(s'|s,a) = \frac{N^k_h(s, a, s')}{N^k_h(s, a)}$ and $c$ is a constant.
\ENDFOR
\end{algorithmic}
\end{algorithm}

Most of this algorithm resembles the typical optimistic planning \citep{jaksch2010near} but with some notable differences. In Line \ref{alg:greedy}, the value function $V_{1, \theta}$ is for the original MDP with transition probabilities parametrized by $\theta$. Different from the typical optimistic planning, the underlying MDP here obeys the stochastic observable model in Assumption~\ref{assumption:missing}. Therefore, the value $V_{1, \theta}$ is the sum of all possible values under missing observations. When counting $N_h^k(s, a)$ in Line \ref{alg:missing_count}, we exclude data tuples missing the next state, which inevitably slows down the learning curve. Nonetheless, the effect of missing only contributes as a scaling factor in the regret.
\begin{proposition}\label{prop:s2arate}
Suppose Assumption~\ref{assumption:missing} holds with $\lambda_h$ known. Given a failure probability $\gamma$, with probability $1 - \gamma$, the regret of Algorithm~\ref{alg:missing} satisfies
\begin{align*}
\subopt(K) \leq c \left(\left\lceil\frac{1}{-\log (1-\lambda_0^2)}\right\rceil \sqrt{H^3 S^2 AK \iota^3} + \sqrt{H^4K\iota}\right),
\end{align*}
where $\iota = \log \frac{SAHK}{\gamma}$ and $c$ is a constant.
\end{proposition}
The proof is provided in Appendix~\ref{pf:s2arate}. Proposition~\ref{prop:s2arate} is optimal in the $K$ dependence and achieves an $S^2 A$ dependence on the complexity of the underlying MDP. In the extreme case of $\lambda_0 \approx 0$, which implies that every state and reward are hardly observable, we have $\subopt(K) = \tilde{\cO}\left(\frac{1}{\lambda_0^2}\sqrt{H^3 S^2 A K}\right)$. Here $\lambda_0^2$ is the probability of observing two consecutive states for estimating the transition probabilities. Proposition~\ref{prop:s2arate} requires the knowledge of observable rate $\lambda_h$. This is not a restrictive condition, as estimating $\lambda_h$ from Bernoulli random variables is much easier than estimating transition probabilities.

\subsection{Model-Based Planning using Augmented MDPs} Proposition~\ref{prop:s2arate} is not sharp in the dependence on $S$. We next show that the augmented MDP approach is effective to tackle missing observations, when the observable rate satisfies additional conditions. Specifically, we assume that the observable rate $\lambda_h$ is independent of $(s, a)$. We utilize the $\mdpaug$ reformulation, except that we redefine the transition probabilities as
\begin{align*}
\paug{h}(\tau_{h+1} | \tau_{h}, a_h) =
\begin{cases}
\lambda_h p_h(s_{h+1} | s_{t_h}, \as{t_h}{h}) & \text{if}~t_{h+1} = h+1 \\
{\tt M}_a(\tau_{h+1}, \tau_{h}) (1-\lambda_h) & \text{if}~t_{h+1} = t_{h} \\
0 & \text{otherwise}
\end{cases}.
\end{align*}
The first case in $\paug{h}$ corresponds to receiving the state observation at time $h+1$. In contrast to the delayed MDPs, the transition probabilities here potentially rely on multi-step transitions in the original MDP. The second case of the transition corresponds to missing the observation. We summarize the policy learning procedure in Algorithm~\ref{alg:missing}.
\begin{algorithm}[h]
\caption{Policy learning for MDPs with missing observations}
\label{alg:missing}
\begin{algorithmic}[1]
\STATE {\bf Input}: Horizon $H$.
\STATE {\bf Init}: $V_{H+1}(\tau) = 0$ and $Q_H(\tau, a) = H$ for any $\tau, a$, data set $\cD^0 = \emptyset$, initial policy $\pi^0$.
\FOR{episode $k = 1, \dots, K$}
\STATE Execute policy $\pi^{k-1}$.
\STATE After the episode ends, collect data $\cD^k = \cD^{k-1} \cup \{(s_h^k, a_h^k, r_h^k)\}_{h=1}^H$.
\STATE On the data set $\cD^k$, compute counting numbers
\begin{align*}
N_h^k(\tau_h, a_h) = \sum_{j=1}^k \mathds{1}\{\tau_h^j = \tau_h, a_h^j = a_h, s_{h+1}^j \neq \emptyset \} \quad \text{and} \quad N_{h, \lambda}^k = \sum_{j=1}^k \mathds{1}\{s_h^j = \emptyset \}.
\end{align*}
\STATE Estimate transition probabilities and delay distributions via
\begin{align*}
\hat{p}_h^k(s_{h+1} | \tau_h, a_h) = \frac{N_h^k(\tau_h, a_h, s_{h+1})}{N_h^k(\tau_{h}, a_h)} \quad \text{and} \quad \hat{\lambda}^k_h = N_{h, \lambda}^k / k.
\end{align*}
\STATE Set the bonus function as
\begin{align*}
b_h^k(\tau_h, a_h) = c H \left(\sqrt{\frac{H\iota}{N_h^k(\tau_h, a_h)}} + \sqrt{\frac{\iota}{k}}\right)
\end{align*}
for $\iota = \log \frac{SAKH}{\gamma}$ and $c$ sufficiently large.
\STATE Run optimistic value iteration in $\mdpaug$ for $H$ steps and obtain $\pi^k \in \pilo$. \label{line:missing_planning}
\ENDFOR
\STATE {\bf Return}: Learned policy $\pi^k$ for $k = 1, \dots, K$.
\end{algorithmic}
\end{algorithm}

We remark that similar to delayed MDPs, in Line~\ref{line:missing_planning} the planning is on $\mdpaug$ and the obtained policy is executable given any $\tau \in \Saug$ when state observation is missed. Therefore, the planning complexity is $SA^H$. Different from Algorithm~\ref{alg:future}, the bonus function here depends on multi-step transitions, in that missing observations are permanently lost. The following theorem shows that Algorithm~\ref{alg:missing} is asymptotically efficient when the observable rate is relatively high.
\begin{theorem}\label{thm:regret_missing}
Suppose Assumption~\ref{assumption:missing} holds with $\lambda_0 \geq 1 - A^{-(1+v)}$ for some positive constant $v$. Given a failure probability $\gamma$, with probability $1 - \gamma$, the regret of Algorithm~\ref{alg:missing} satisfies
\begin{align*}
\subopt(K) \leq c \left(H^4 \sqrt{SAK \iota^3} + S^2 \sqrt{H^9 K^{\frac{1}{(1+v)}} \iota^6}\right),
\end{align*}
where $\iota = \log \frac{SAHK}{\gamma}$ and $c$ is a constant.
\end{theorem}

The proof is provided in Appendix~\ref{pf:regret_missing}. Some remarks are in order.

\paragraph{$SA$ rate when $K$ is large} When the number of episodes $K \geq S^{3(1+v)/v}$, the first term $H^4\sqrt{SAK\iota^3}$ in the regret bound dominates and attains a sharp dependence on $S$ and $A$. However, when the number of episodes are limited, the regret bound has a worse dependence on the state space size $S$. We also observe that as the missing rate $\lambda$ becomes small (equivalently, $v$ becomes large), the regret is close to $\tilde{O}(H^4 \sqrt{SAK\iota^3})$.

\paragraph{Observable rate smaller than $1 - 1/A$} Theorem~\ref{thm:regret_missing} holds for an observable rate $\lambda_0 > 1 - 1/A$. The intuition behind this is that, to fully explore all the actions when a state observation is missing takes $A$ trials. Therefore, in expectation, we will encounter a missing observation at least every $A$ episodes as long as $\lambda_0 > 1 - 1/A$. Nonetheless, when $\lambda_0 \leq 1 - 1/A$, the regret bound remains curiously underexplored. We conjecture that $\lambda_0 = 1 - 1/A$ is a critical point distinguishing unique strategies for learning and planning in MDPs with missing observations. A detailed analysis of this issue goes beyond the scope of the current paper.

\paragraph{Proof sketch} The proof of Theorem~\ref{thm:regret_missing} adapts the analysis of model-based UCBVI algorithms \cite{azar2017minimax}. Let $m$ denote the maximal length of consecutive missing observations. We denote $\cE_m$ as the event when the maximal length of consecutive missing is less than $m$. On event $\cE_m$, a na\"{i}ve analysis leads to a $\tilde{\cO}\left(\sqrt{{\rm poly}(H)SA^{m+1} K}\right)$ regret, in observation to the size of the expanded state space $\Saug$. However, our analysis circumvents the $A^m$ dependence by exploiting the occurrence of consecutive missing observations is rare (Lemma~\ref{lemma:max_gap}). On the complement of event $\cE_m$, the regret is bounded by $KH(1 - \PP(\cE_m))$. Summing up the two parts and choosing a proper $m$ yield our result.

\section{Numerical Experiment with Delayed Observations}\label{sec:experiment}

This section presents synthetic data experiments to validate our theory on MDPs with delayed observations.

\paragraph{Tabular MDP instance} We focus on a 20-state 5-action MDP instance, i.e., $S = 20$ and $A = 5$. The horizon is set to be $H = 10$ and we vary the maximal length of delay in $\{0, 1, 2, 3, 4\}$. We use $i \in \{1, \dots, S\}$ to index the state and $j \in \{1, \dots, A\}$ to index the action. The reward function is set as 
\begin{align*}
r_h(s_i, a_j) = \mathds{1}\{i = S/2~\text{or}~ S/2 + 1\}, \quad \text{for}~ 1 \leq h \leq H, ~1 \leq j \leq A,
\end{align*}
which implies that the optimal policy should maximally maintain a large visitation measure on state $s_{S/2}$ and $s_{S/2+1}$ to receive a nonzero reward.

We specify two sets of transition probabilities being time-homogeneous, corresponding to nearly deterministic and more stochastic transitions, respectively. Specifically, for nearly deterministic transitions (\texttt{Env1}), we set $a_1$ as a special action that moves states towards the nonzero reward state $s_{S/2}$ and $s_{S/2+1}$. The transition probabilities are
\begin{equation}\label{eq:env1_transition_a1}
\begin{split}
& p_h(s_{i+1}|s_i, a_1) = 0.99, \qquad \text{for} ~ 1 \leq h \leq H, ~1 \leq i \leq \frac{S}{2}, \\
& p_h(s_{i-1}|s_i, a_1) = 0.01, \qquad \text{for} ~ 1 \leq h \leq H, ~1 < i \leq \frac{S}{2}, \\
& p_h(s_{1}|s_1, a_1) = 0.01, \qquad \text{for} ~ 1 \leq h \leq H, \\
& p_h(s_{i-1}|s_i, a_1) = 0.99, \qquad  \text{for} ~ 1 \leq h \leq H, ~\frac{S}{2}+1 \leq i \leq S, \\
& p_h(s_{i+1}|s_i, a_1) = 0.01, \qquad \text{for} ~ 1 \leq h \leq H, ~\frac{S}{2}+1 \leq i < S, \\
& p_h(s_{S}|s_S, a_1) = 0.01, \qquad \text{for} ~ 1 \leq h \leq H.
\end{split}
\end{equation}
As opposed to the special action $a_1$, all the other actions leads the state to either $s_1$ or $s_S$, which will not receive a reward. The transition probabilities are
\begin{equation}\label{eq:env1_transition_other}
\begin{split}
& p_h(s_{i+1}|s_i, a_j) = 0.01, \qquad \text{for} ~ 1 \leq h \leq H, ~1 \leq i \leq \frac{S}{2}, ~2 \leq j \leq A, \\
& p_h(s_{i-1}|s_i, a_j) = 0.99, \qquad \text{for} ~ 1 \leq h \leq H, ~1 < i \leq \frac{S}{2}, ~2 \leq j \leq A, \\
& p_h(s_{1}|s_1, a_j) = 0.99, \qquad \text{for} ~ 1 \leq h \leq H, ~2 \leq j \leq A, \\
& p_h(s_{i-1}|s_i, a_j) = 0.01, \qquad \text{for} ~ 1 \leq h \leq H, ~\frac{S}{2}+1 \leq i \leq S, ~2 \leq j \leq A, \\
& p_h(s_{i+1}|s_i, a_j) = 0.99, \qquad \text{for} ~ 1 \leq h \leq H, ~\frac{S}{2}+1 \leq i < S, ~2 \leq j \leq A, \\
& p_h(s_{S}|s_S, a_j) = 0.99, \qquad \text{for} ~ 1 \leq h \leq H, ~2 \leq j \leq A.
\end{split}
\end{equation}
Similar to (\texttt{Env1}), for a more stochastic environment (\texttt{Env2}), we still set $a_1$ as a special action and all the other actions as being relatively ``bad''. Yet we increase the randomness in each state jump, i.e., we replace $0.99$ in \eqref{eq:env1_transition_a1} and \eqref{eq:env1_transition_other} by $0.8$ and $0.01$ by $0.2$.
The transition structure is illustrated in Figure~\ref{fig:exp_instance}. Clearly, in ({\texttt{Env2}}), we are less certain of the next state given the previous one. We remark that the choice of reward and transition are for illustration purposes. Any specification leads to the same observation presented in the sequel.
\begin{figure}[htb!]
\centering
\includegraphics[width = 0.55\textwidth]{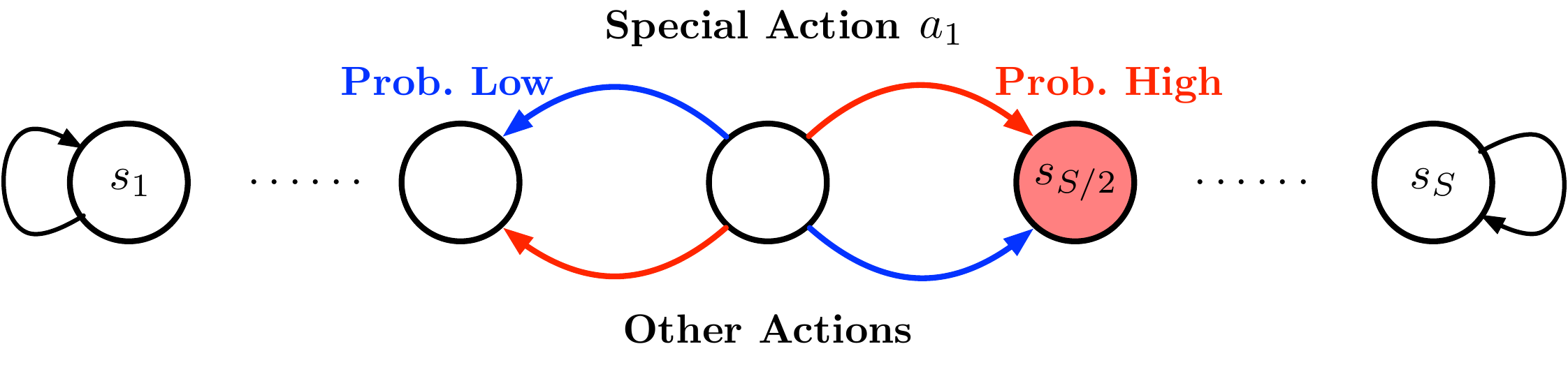}
\caption{Transition structure of the 20-state 5-action MDP instance. Each state will transit to its neighbors in the next step. The spacial action $a_1$ guides the transition towards nonzero reward states.}
\label{fig:exp_instance}
\end{figure}

\paragraph{Delay distribution} We consider both constant delays and stochastic delays. In particular, the constant delay varies with a length of $0$ (no delay), $1$ or $2$. For stochastic delay, we specify the interarrival time $\Delta_h$ following a Bernoulli distribution with $\Delta_h \in \{0, 1\}$ for horizon $h = 1, 2, 3$. Afterwards, we set $\Delta_h = 0$ for horizon $h > 3$. As a result, in the stochastic setting, the maximal length of delay is bounded by $3$. We vary the probability of $\PP(\Delta_h = 0)$ to be $0.9, 0.7, 0.5$ or $0.3$.

\paragraph{Regret plots and gap table} Curves of regret against the number of episodes are presented in Figure~\ref{fig:experiment}. Each curve is averaged over 3 independent runs with random seeds. The shaded area around the line represents the error bars (i.e., the standard deviation). In Figure~\ref{fig:regret_env1} and \ref{fig:regret_env2}, we provide the regret in constant delay (solid lines) and stochastic delay (broken lines) settings under (\texttt{Env1}) and (\texttt{Env2}), respectively. As can be seen, as the length of delay increases, the obtained regret also increases, supporting our regret bound in Theorem~\ref{thm:delay_regret}. In addition, when the expected length of delay increases in the stochastic delay setting, the regret also increases. 

We demonstrate the performance degradation caused by delayed observations by comparing \gap in (\texttt{Env1}) and (\texttt{Env2}). We obtain $\pinodelay$ by value iteration in the original MDP and $\pidelay$ by value iteration in $\tmdpaug$. For a convenient deployment of $\pinodelay$, we randomly initialize the MDP from a past time. For example, in the constant delay setting, we set $s_{1-D}$ uniformly sampled from the state space and actions $a_{1-D}, \dots, a_{0}$ all uniform in the action space. With random initialization, we compute the corresponding value functions $V^{\pinodelay}_1$ and $V^{\pidelay}_1$ averaged over $10000$ episodes. This procedure is further independently repeated for 5 times to compute the standard deviation. The results are reported in Table~\ref{tab:gap}. As can be seen, (\texttt{Env2}) introduces more randomness in the transition probabilities and therefore, we observe that under the same delay pattern, \texttt{gap} is larger in (\texttt{Env2}) than that in (\texttt{Env1}). Moreover, as the length of delay increases, \texttt{gap} also increases, due to enlarged uncertainty to infer the immediate state. These observations corroborate the discovery in Proposition~\ref{prop:pidelay_gap} and Proposition~\ref{prop:d+1_vs_d}.

\begin{figure}[!htb]
\centering
\begin{subfigure}[b]{0.48\textwidth}
\centering
\includegraphics[width = \textwidth]{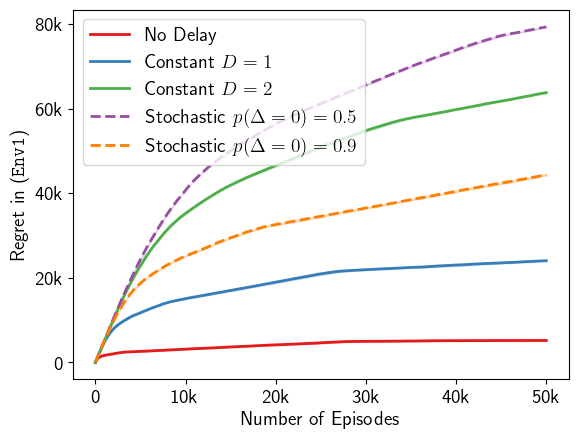}
\caption{Regret as a function of number of episodes in (\texttt{Env1}).}
\label{fig:regret_env1}
\end{subfigure}
~
\begin{subfigure}[b]{0.48\textwidth}
\centering
\includegraphics[width = \textwidth]{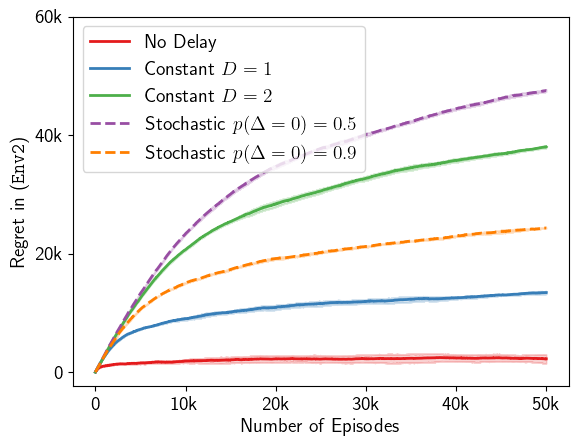}
\caption{Regret as a function of number of episodes in (\texttt{Env2}). }
\label{fig:regret_env2}
\end{subfigure}
\hfill
\vspace{0.05in}
\caption{\subopt~in a 20-state 5-action tabular MDP setting. Nearly deterministic transition (\texttt{Env1}) and more random transition (\texttt{Env2}) are tested with various delay distributions.}
\label{fig:experiment}
\end{figure}

\begin{table}[htb!]
\centering
\caption{Value and standard deviation of \texttt{Gap} in a nearly deterministic environment (\texttt{Env1}) and a more stochastic environment (\texttt{Env2}).}
\label{tab:gap}
\begin{tabular}{ c | c | c | c | c | c | c | c | c }
\hline
\multirow{2}{*}{\texttt{Gap}~($\times 10^{-4}$)} & \multicolumn{4}{c|}{Constant Delay} & \multicolumn{4}{c}{Stochastic Delay} \\\cline{2-9}
& $D = 1$ & $D = 2$ & $D = 3$ & $D = 4$ & $p = 0.9$ & $p = 0.7$ & $p = 0.5$ & $p = 0.3$ \\\hline
(\texttt{Env1})  & 18$\pm$8 & 1361$\pm$19 & 2083$\pm$15 & 2124$\pm$24 & 185$\pm$49 & 501$\pm$50 & 684$\pm$17 & 808$\pm$16 \\\hline
(\texttt{Env2}) & 71$\pm$16 & 1781$\pm$69 & 3232$\pm$11 & 3280$\pm$11 & 185$\pm$97 & 636$\pm$57 & 882$\pm$50 & 1185$\pm$32 \\\hline
\end{tabular}
\end{table}

\section{Conclusion}\label{sec:conclusion}

In this paper, we have studied learning and planning in MDPs with impaired observability. We have focused on MDPs with delayed and missing observations. Specifically, for delayed observations, we have proposed algorithms and shown an efficient $\tilde{O}((H \wedge D)^{5/2} \sqrt{H^3SAK})$ regret. For missing observations, we have provided an optimistic planning algorithm achieving an $\tilde{O}(\sqrt{H^3 S^2AK})$ regret. If the missing rate is relatively small, we have established an efficient $\tilde{O}(H^4\sqrt{SAK})$ regret bound. Further, we have characterized the performance degradation caused by impaired observability compared to full observability. Numerical results corroborate the theoretical findings.

\bibliography{ref, ref2}
\bibliographystyle{plainnat}

\appendix

\section{Omitted proof in Section \ref{sec:alg}} \label{pf:mdp=mdptilde}
\begin{proof}[Proof of Proposition~\ref{prop:mdp=mdptilde}]
Consider an arbitrary fixed inter-arrival pattern $\Delta_0, \Delta_1, \dots, \Delta_{H-1}$. We show that the expected accumulated rewards under this inter-arrival pattern are identical for $\mdpaug$ and $\tmdpaug$. In $\tmdpaug$, we have
\begin{align*}
& ~\quad \EE^\pi \left[ \sum_{h=1}^{\tilde{H}} \traug{h}(\tau_h, a_h) ~\Big|~ \tau_1, \Delta_0, \dots, \Delta_{H-1} \right] \\
& \overset{(i)}{=} \EE^\pi \left[\sum_{h=1}^{\tilde{H}} \traug{t_h}(s_{t_h}, a_{t_h}) \mathds{1}\{\delta_{t_h} = 0\}\mathds{1}\{t_h \in \{1, \dots, H\}\} ~\Big|~ \tau_1, \Delta_0, \dots, \Delta_{H-1} \right]\\
& \overset{(ii)}{=} \EE^\pi \left[ \sum_{h=1}^H r(s_{h}, a_{h}) ~\Big|~ \tau_1, \Delta_0, \dots, \Delta_{H-1} \right] \\
& = \EE^\pi \left[ \sum_{h=1}^H \raug{h}(\tau_{h}, a_{h}) ~\Big|~ \tau_1, \Delta_0, \dots, \Delta_{H-1} \right],
\end{align*}
where equality $(i)$ invokes the definition of $\traug{h}$ and equality $(ii)$ eliminates zero reward terms. Now taking expectation over all possible inter-arrival patterns, we deduce
\begin{align*}
\EE^\pi \left[\sum_{h=1}^{\tilde{H}} \traug{h}(\tau_h, a_h) ~\Big|~ \tau_1 \right] = \EE^\pi \left[ \sum_{h=1}^H \raug{h}(s_{h}, a_{h}) ~\Big|~ \tau_1 \right].
\end{align*}
The proof is complete.
\end{proof}

\section{Omitted proofs in Section \ref{sec:regret_delay}}

\subsection{Proof of Theorem \ref{thm:delay_regret}}\label{pf:delay_regret}
\begin{proof}
We adapt the main steps from \cite{azar2017minimax} for proving the theorem. The proof consists of verifying a valid optimism and developing a regret analysis. We denote $\tQaug{h}^*$ as the optimal $Q$-function for $\tmdpaug$. When analyzing the regret, we also denote $\tilde{Q}_{h, {\rm aug}}^k$ as the optimal $Q$-function in the $k$-th episode.
\paragraph{Valid optimism} To begin with, we verify that the choice of the bonus functions leads to a valid optimism in the following lemma.
\begin{lemma}\label{lemma:optimism}
Given any failure probability $\gamma < 1$, we set a bonus as
\begin{align*}
b_h^k(\tau_h, a_h) = c_A H \left(\sqrt{\frac{(H \wedge D) \iota}{N_{t_h}(s_{t_h}, a_{t_h}, \delta_{t_h})}} + \sqrt{\frac{(H \wedge D) \iota}{N_{t_h}(s_{t_h}, a_{t_h})}}\right),
\end{align*}
where $\iota = \log \left(\frac{SAHK}{\gamma}\right)$ and $c_A$ is a constant. Then with probability $1 - \gamma$, it holds that
\begin{align*}
\tQaug{h}^k(\tau_h, a_h) \geq \tQaug{h}^*(\tau_h, a_h), \quad \tVaug{h}^k(\tau_h) \geq \tVaug{h}^*(\tau_h) \quad \text{for~any}\quad (k, h, \tau_h, a_h).
\end{align*}
\end{lemma}
\begin{proof}[Proof of Lemma~\ref{lemma:optimism}]
We compute the cardinality of the expanded state space $\tSaug$ as
\begin{align*}
|\tSaug| \overset{(i)}{=} \sum_{i=0}^{H \wedge D} H S A^i = H S \frac{A^{(H \wedge D) + 1} - 1}{A - 1} \leq 2 H S A^{H \wedge D}.
\end{align*}
For a fixed episode $k$, we show by backward induction that the assertion in Lemma~\ref{lemma:optimism} holds. To ease the presentation, we omit all superscripts $k$, all subscripts ``aug'', as well as the tilde $\tilde{\cdot}$ notation and subscript ``delay'' in $\theta$. When $h = \tilde{H}+1$, the base assertion holds immediately. Suppose the assertion is true for time $h+1$. At time $h$, for any fixed $(\tau_h, a_h)$, if $Q_h(\tau_h, a_h) = H$, the assertion holds true. Otherwise, we have
\begin{align*}
Q_h(\tau_h, a_h) - Q^*_h(\tau_h, a_h) & = [\hat{\cP}_hV_{h+1}](\tau_h, a_h) - [\cP_h V_{h+1}^*](\tau_h, a_h) + b_h^k(\tau_h, a_h) \\
& \geq \underbrace{\left([\hat{\cP}_h - \cP_h]V_{h+1}^*\right)(\tau_h, a_h)}_{(A)} + ~b_h^k(\tau_h, a_h).
\end{align*}
We show a lower bound on $(A)$. If $h \ge H$, expanding the transition kernel $\cP_h$ leads to
\begin{align*}
(A) & = \sum_{\tau_{h+1}} V_{h+1}^*(\tau_{h+1}) (\hat{p}_h(\tau_{h+1} | \tau_h, a_h) - p_h(\tau_{h+1} | \tau_{h}, a_h)) \\
& \overset{(i)}{=} \sum_{s_{t_h+1}} V_{h+1}^*(\tau_{h+1}) (\hat{p}_h(s_{t_h+1} | s_{t_h}, a_{t_h}) - p_h(s_{t_h+1} | s_{t_h}, a_{t_h})) \\
& \overset{(ii)}{\geq} - c_{A, 1} H \sqrt{\frac{(H \wedge D) \iota}{N_{t_h}(s_{t_h}, a_{t_h})}},
\end{align*}
where equality $(i)$ requires $\tau_{h+1}$ to take $s_{t_h+1}$ as the new state observation, and inequality $(ii)$ follows from Hoeffding's inequality (Lemma~\ref{lemma:hoeffding}) with a constant $c_{A, 1}$. Note that the $(H \wedge D) \iota$ term in the numerator comes from a union bound over $\tSaug \times \cA$.

On the other hand, if $h < H$, expanding the transition kernel $\cP_h$ yields
\begin{align*}
(A) & = \sum_{\tau_{h+1}} V_{h+1}^*(\tau_{h+1}) \left(\hat{p}_h(\tau_{h+1} | \tau_h, a_h) - p_h(\tau_{h+1} | \tau_h, a_h) \right) \\
& = \underbrace{\sum_{\tau_{h+1}} V_{h+1}^*(\tau_{h+1}) \left(\hat{p}_h(\tau_{h+1} | \tau_h, a_h) - p_h(\tau_{h+1} | \tau_h, a_h) \right) \mathds{1}\{\delta_{t_{h+1}} = 0\}\mathds{1}\{t_{h+1} = t_{h}+1\}}_{(A_1)} \\
& \quad + \underbrace{\sum_{\tau_{h+1}} V_{h+1}^*(\tau_{h+1}) \left(\hat{p}_h(\tau_{h+1} | \tau_h, a_h) - p_h(\tau_{h+1} | \tau_h, a_h) \right) \mathds{1}\{\delta_{t_{h+1}} = \delta_{t_h}+1\}\mathds{1}\{t_{h+1} = t_h\}}_{(A_2)}.
\end{align*}
Note that $(A_1)$ accounts for receiving a new state observation in $\tau_{h+1}$, and $(A_2)$ accounts for no new state observation. We tackle these two terms separately. For $(A_1)$, we have
\begin{align*}
(A_1) & = \sum_{s_{t_{h+1}}} V_{h+1}^*(\tau_{h+1}) \left((1-\hat{\theta}_{t_h}(s_{t_h}, a_{t_h}, \delta_{t_h}))\hat{p}_{t_h}(s_{t_{h+1}} | s_{t_h}, a_{t_h}) - (1-\theta_{t_h}(s_{t_h}, a_{t_h}, \delta_{t_h}))p_{t_h}(s_{t_{h+1}} | s_{t_h}, a_{t_h}) \right) \\
& = \sum_{s_{t_{h+1}}} V_{h+1}^*(\tau_{h+1}) \left(\left(1-\hat{\theta}_{t_h}(s_{t_h}, a_{t_h}, \delta_{t_h})\right) - \left(1-\theta_{t_h}(s_{t_h}, a_{t_h}, \delta_{t_h})\right)\right)\hat{p}_{t_h}(s_{t_{h+1}} | s_{t_h}, a_{t_h}) \\
& \quad + \sum_{s_{t_{h+1}}} V_{h+1}^*(\tau_{h+1}) (1-\theta_{t_h}(s_{t_h}, a_{t_h}, \delta_{t_h}))\left(\hat{p}_{t_h}(s_{t_{h+1}} | s_{t_h}, a_{t_h}) - p_{t_h}(s_{t_{h+1}} | s_{t_h}, a_{t_h}) \right) \\
& \overset{(i)}{\geq} - H \left| \hat{\theta}_{t_h}(s_{t_h}, a_{t_h}, \delta_{t_h}) - \theta_{t_h}(s_{t_h}, a_{t_h}, \delta_{t_h}) \right| - c_{A, 2} H \sqrt{\frac{(H \wedge D)\iota}{N_{t_h}(s_{t_h}, a_{t_h})}},
\end{align*}
where in $(i)$, the first term is the estimation error of $\hat{\theta}$ using the collected data, the second term follows from Hoeffding's inequality, and $c_{A, 2}$ is an absolute constant. For $(A_2)$, we have
\begin{align*}
(A_2) \geq - H \left| \hat{\theta}_{t_h}(s_{t_h}, a_{t_h}, \delta_{t_h}) - \theta_{t_h}(s_{t_h}, a_{t_h}, \delta_{t_h}) \right|,
\end{align*}
since $\tau_{h+1}$ is now uniquely determined. Summing up $(A_1)$ and $(A_2)$, we obtain
\begin{align*}
(A) = (A_1) + (A_2) \geq - 2H \left| \hat{\theta}_{t_h}(s_{t_h}, a_{t_h}, \delta_{t_h}) - \theta_{t_h}(s_{t_h}, a_{t_h}, \delta_{t_h}) \right| - c_{A, 2} H \sqrt{\frac{(H \wedge D)\iota}{N_{t_h}(s_{t_h}, a_{t_h})}}.
\end{align*}
It remains to bound the estimation error of $\hat{\theta}_{t_h}(s_{t_h}, a_{t_h}, \delta_{t_h})$. Using Hoeffding's inequality again, we obtain
\begin{align*}
\left| \hat{\theta}_{t_h}(s_{t_h}, a_{t_h}, \delta_{t_h}) - \theta_{t_h}(s_{t_h}, a_{t_h}, \delta_{t_h}) \right| \leq c_{\theta} \sqrt{\frac{(H \wedge D) \iota}{N_{t_h}(s_{t_h}, a_{t_h}, \delta_{t_h})}}.
\end{align*}
Taking $c_{A} = \max\{c_{A, 1}, c_{A, 2}, c_\theta, 2\}$, we have
\begin{align*}
(A) \geq - c_A H \left(\sqrt{\frac{(H \wedge D) \iota}{N_{t_h}(s_{t_h}, a_{t_h}, \delta_{t_h})}} + \sqrt{\frac{(H \wedge D)\iota}{N_{t_h}(s_{t_h}, a_{t_h})}}\right).
\end{align*}
With the choice of the bonus function, it can be checked that
\begin{align*}
\tQaug{h}^k(\tau_h, a_h) - \tQaug{h}^*(\tau_h, a_h) \geq (A) + b_h^k(\tau_h, a_h) \geq 0
\end{align*}
with probability $1 - \gamma$ for any $(\tau_h, a_h)$.
\end{proof}

\paragraph{Regret analysis}
In the sequel, we omit subscripts ``aug'' and ``delay'' as well as tilde $\tilde{\cdot}$ for simplicity. Thanks to Lemma~\ref{lemma:optimism}, we consider $\left(Q_h^k - Q_h^{\pi_k}\right)(\tau_h^k, a_h^k)$ as an upper bound of $\left(Q_h^* - Q_h^{\pi_k}\right)(\tau_h^k, a_h^k)$. We bound $\left(Q^k_h - Q^{\pi_k}_h\right)(\tau_h^k, a_h^k)$ as
\begin{align}\label{eq:delay_regret_decomp}
& \quad \left(Q^k_h - Q^{\pi_k}_h\right)(\tau_h^k, a_h^k) \nonumber \\
& \leq \left([\hat{\cP}_h^k V_{h+1}^k - \cP_h V_{h+1}^{\pi_k}]\right)(\tau_h^k, a_h^k) + b_h^k(\tau_{h}^k, a_{h}^k) \nonumber \\
& \leq \left([\hat{\cP}_h^k - \cP_h]V_{h+1}^*\right)(\tau_h^k, a_h^k) + \left([\hat{\cP}_h^k - \cP_h][V^k_{h+1} - V_{h+1}^*]\right)(\tau_h^k, a_h^k) \nonumber \\
& \quad + \left(\cP_h[V_{h+1}^k - V_{h+1}^{\pi_k}]\right)(\tau_h^k, a_h^k) + b_h^k(\tau_{h}, a_{h}^k) \nonumber \\
& \leq \underbrace{\left([\hat{\cP}_h^k - \cP_h][V^k_{h+1} - V_{h+1}^*]\right)(\tau_h^k, a_h^k)}_{(A)} + \left(\cP_h[V_{h+1}^k - V_{h+1}^{\pi_k}]\right)(\tau_h^k, a_h^k) + 2b_h^k(\tau_{h}^k, a_{h}^k).
\end{align}
Similar to Lemma~\ref{lemma:optimism}, for $h \ge H$, we expand term $(A)$ into
\begin{align}\label{eq:h>H}
(A) & = \sum_{\tau_{h+1}} \left(\hat{p}_h^k(\tau_{h+1} | \tau_{h}^k, a_h^k) - p_h(\tau_{h+1} | \tau_{h}^k, a_h^k)\right) [V_{h+1}^k - V_{h+1}^*](\tau_{h+1}) \nonumber \\
& = \sum_{s_{t_h+1}} [V_{h+1}^k - V_{h+1}^*](\tau_{h+1}) \left(\hat{p}_{t_h}^k(s_{t_h+1} | s_{t_h}^k, a_{t_h}^k) - p_{t_h}(s_{t_h+1} | s_{t_h}^k, a_{t_h}^k) \right).
\end{align}
In the second equality of \eqref{eq:h>H}, $\hat{p}_{t_h}^k$ refers to the transition in the original MDP, as it takes state $s$ as input. On the other hand, for $h \leq H$, the decomposition of term $(A)$ is more complicated. We have
\begin{align*}
(A) & = \sum_{\tau_{h+1}} \left(\hat{p}_h^k(\tau_{h+1} | \tau_{h}^k, a_h^k) - p_h(\tau_{h+1} | \tau_{h}^k, a_h^k)\right) [V_{h+1}^k - V_{h+1}^*](\tau_{h+1}) \\
& = \underbrace{\sum_{\tau_{h+1}} [V_{h+1}^k - V_{h+1}^*](\tau_{h+1}) \left(\hat{p}_h^k(\tau_{h+1} | \tau_h^k, a_h^k) - p_h(\tau_{h+1} | \tau_h^k, a_h^k) \right) \mathds{1}\{\delta_{t_{h+1}} = 0\}\mathds{1}\{t_{h+1} = t_{h}^k+1\}}_{(A_1)} \\
& \quad + \underbrace{\sum_{\tau_{h+1}} [V_{h+1}^k - V_{h+1}^*](\tau_{h+1}) \left(\hat{p}_h^k(\tau_{h+1} | \tau_h^k, a_h^k) - p_h(\tau_{h+1} | \tau_h^k, a_h^k) \right) \mathds{1}\{\delta_{t_{h+1}} = \delta_{t_h^k}+1\}\mathds{1}\{t_{h+1} = t_h^k\}}_{(A_2)}.
\end{align*}
The term $(A_2)$ can be directly bounded as
\begin{align*}
(A_2) & \leq H \left|\hat{\theta}^k_{t_h}(s_{t_h}^k, a_{t_h}^k, \delta_{t_h}^k) - \theta_{t_h}(s_{t_h}^k, a_{t_h}^k, \delta_{t_h}^k)\right| \\
& \leq c_{\theta} H \sqrt{\frac{(H \wedge D) \iota}{N_{t_h}^k(s_{t_h}^k, a_{t_h}^k, \delta_{t_h}^k)}}
\end{align*}
with probability $1 - \gamma$. To bound $(A_1)$, we have
\begin{align}
(A_1) & = \sum_{s_{t_{h+1}}} [V_{h+1}^k - V_{h+1}^*](\tau_{h+1}) \bigg(\left(1-\hat{\theta}^k_{t_h}(s_{t_h}^k, a_{t_h}^k, \delta_{t_h}^k)\right)\hat{p}_{t_h}^k(s_{t_{h+1}} | s_{t_h}^{k}, a_{t_h}^{k}) \nonumber \\
& \quad - \left(1-\theta_{t_h}(s_{t_h}^k, a_{t_h}^k, \delta_{t_h}^k)\right)p_{t_h}(s_{t_{h+1}} | s_{t_h}^{k}, a_{t_h}^{k}) \bigg) \nonumber \\
& = \sum_{s_{t_{h+1}}} [V_{h+1}^k - V_{h+1}^*](\tau_{h+1}) \left(\left(1-\hat{\theta}^k_{t_h}(s_{t_h}^k, a_{t_h}^k, \delta_{t_h}^k)\right) - \left(1-\theta_{t_h}(s_{t_h}^k, a_{t_h}^k, \delta_{t_h}^k)\right)\right)\hat{p}_{t_h}^k(s_{t_{h+1}} | s_{t_h}^{k}, a_{t_h}^{k}) \nonumber \\
& \quad + \sum_{s_{t_{h+1}}} [V_{h+1}^k - V_{h+1}^*](\tau_{h+1}) \left(1-\theta_{t_h}(s_{t_h}^k, a_{t_h}^k, \delta_{t_h}^k)\right)\left(\hat{p}_{t_h}^k(s_{t_{h+1}} | s_{t_h}^{k}, a_{t_h}^{k}) - p_{t_h}(s_{t_{h+1}} | s_{t_h}^{k}, a_{t_h}^{k}) \right) \nonumber \\
& \leq \left(1-\theta_{t_h}(s_{t_h}^k, a_{t_h}^k, \delta_{t_h}^k)\right)\sum_{s_{t_{h+1}}} [V_{h+1}^k - V_{h+1}^*](\tau_{h+1}) \left(\hat{p}^k_{t_h}(s_{t_{h+1}} | s_{t_h}^{k}, a_{t_h}^{k}) - p_{t_h}(s_{t_{h+1}} | s_{t_h}^{k}, a_{t_h}^{k}) \right)\nonumber \\
& \quad + H \left|\hat{\theta}^k_{t_h} (s_{t_h}^k, a_{t_h}^k, \delta_{t_h}^k) - \theta_{t_h}(s_{t_h}^k, a_{t_h}^k, \delta_{t_h}^k)\right| \nonumber \\
& \leq \left(1-\theta_{t_h}(s_{t_h}^k, a_{t_h}^k, \delta_{t_h}^k)\right)\sum_{s_{t_{h+1}}} [V_{h+1}^k - V_{h+1}^*](\tau_{h+1}) \left(\hat{p}^k_{t_h}(s_{t_{h+1}} | s_{t_h}^{k}, a_{t_h}^{k}) - p_{t_h}(s_{t_{h+1}} | s_{t_h}^{k}, a_{t_h}^{k}) \right) \nonumber \\
& \quad + c_{\theta} H \sqrt{\frac{(H \wedge D) \iota}{N_{t_h}^k(s_{t_h}^k, a_{t_h}^k, \delta_{t_h}^k)}}. \nonumber
\end{align}
Putting $(A_1)$ and $(A_2)$ together, we obtain
\begin{align}\label{eq:h<H}
(A) & \leq \left(1-\theta_{t_h}(s_{t_h}^k, a_{t_h}^k, \delta_{t_h}^k)\right) \sum_{s_{t_{h+1}}} [V_{h+1}^k - V_{h+1}^*](\tau_{h+1}) \left(\hat{p}^k_{t_h}(s_{t_{h+1}} | s_{t_h}^{k}, a_{t_h}^{k}) - p_{t_h}(s_{t_{h+1}} | s_{t_h}^{k}, a_{t_h}^{k}) \right) \nonumber \\
& \quad + 2 c_{\theta} H \sqrt{\frac{(H \wedge D) \iota}{N_{t_h}^k(s_{t_h}^k, a_{t_h}^k, \delta_{t_h}^k)}}.
\end{align}
In both \eqref{eq:h>H} and \eqref{eq:h<H} for different ranges of $h$, we apply Bernstein inequality (Lemma~\ref{lemma:bernstein}) to derive
\begin{align}\label{eq:A_bernstein}
& \quad~ \sum_{s_{t_{h+1}}} [V_{h+1}^k - V_{h+1}^*](\tau_{h+1}) \left(\hat{p}_{t_h}^k(s_{t_{h+1}} | s_{t_h}^{k}, a_{t_h}^{k}) - p_{t_h}(s_{t_{h+1}} | s_{t_h}^{k}, a_{t_h}^{k}) \right) \nonumber \\
& \leq c \cdot \sum_{s_{t_{h+1}}} [V_{h+1}^k - V_{h+1}^*](\tau_{h+1}) \left[\sqrt{\frac{p_{t_h}(s_{t_{h+1}} | s_{t_h}^k, a_{t_h}^k) \iota}{N_{t_h}^k(s_{t_h}^k, a_{t_h}^k)}} + \frac{\iota}{N_{t_h}^k(s_{t_h}^k, a_{t_h}^k)}\right] \nonumber \\
& \overset{(i)}{\leq} c \cdot \sum_{s_{t_{h+1}}} [V_{h+1}^k - V_{h+1}^*](\tau_{h+1}) \left[\frac{p_{t_h}(s_{t_{h+1}} | s_{t_h}^k, a_{t_h}^k)}{2cH} + \frac{(2cH+1)\iota}{N_{t_h}^k(s_{t_h}^k, a_{t_h}^k)}\right] \nonumber \\
& \leq c \cdot \left(\frac{SH(2cH+1)\iota}{N_{t_h}^k(s_{t_h}^k, a_{t_h}^k)} + \frac{1}{2cH} \sum_{s_{t_{h+1}}} [V_{h+1}^k - V_{h+1}^*](\tau_{h+1}) p_{t_h}(s_{t_{h+1}} | s_{t_h}^k, a_{t_h}^k)\right),
\end{align}
where inequality $(i)$ follows from $\sqrt{ab} \leq a + b$. Substituting \eqref{eq:A_bernstein} into \eqref{eq:h>H}, for $h \ge H$, we deduce
\begin{align*}
(A) & \leq \frac{1}{2H} \sum_{s_{t_{h+1}}} [V_{h+1}^k - V_{h+1}^*](\tau_{h+1}) p_{t_h}(s_{t_{h+1}} | s_{t_h}^k, a_{t_h}^k) + \frac{cSH(2cH+1)\iota}{N_{t_h}^k(s_{t_h}^k, a_{t_h}^k)} \\
& \overset{(i)}{\leq} \frac{1}{2H} \left(\cP_h[V_{h+1}^k - V_{h+1}^{\pi_k}]\right)(\tau_h^k, a_h^k) + c' \frac{SH^2 \iota}{N_{t_h}^k(s_{t_h}^k, a_{t_h}^k)},
\end{align*}
where $c'$ is a sufficiently large constant. By the same reasoning, substituting \eqref{eq:A_bernstein} into \eqref{eq:h<H}, for $h < H$, we have
\begin{align*}
(A) & \leq \frac{1}{2H} \left(1-\theta_{t_h}(s_{t_h}^k, a_{t_h}^k, \delta_{t_h}^k)\right) \sum_{s_{t_{h+1}}} [V_{h+1}^k - V_{h+1}^*](\tau_{h+1}) p_{t_h}(s_{t_{h+1}} | s_{t_h}^k, a_{t_h}^k) + \frac{cSH(2cH+1)\iota}{N_{t_h}^k(s_{t_h}^k, a_{t_h}^k)} \\
& \quad + 2 c_{\theta} H \sqrt{\frac{(H \wedge D)\iota}{N_{t_h}^k(s_{t_h}^k, a_{t_h}^k, \delta_{t_h}^k)}} \\
& \overset{(i)}{\leq} \frac{1}{2H} \left(\cP_h[V_{h+1}^k - V_{h+1}^{\pi_k}]\right)(\tau_h^k, a_h^k) + c' \frac{SH^2 \iota}{N_{t_h}^k(s_{t_h}^k, a_{t_h}^k)} + 2 c_{\theta} H \sqrt{\frac{(H \wedge D) \iota}{N_{t_h}^k(s_{t_h}^k, a_{t_h}^k, \delta_{t_h}^k)}}.
\end{align*}
We denote $\zeta_h^k = c' \frac{SH^2 \iota}{N_{t_h}^k(s_{t_h}^k, a_{t_h}^k)}$. Now we have a unified upper bound on $(A)$ for any $h \in [1, \tilde{H}]$ as
\begin{align}\label{eq:delay_regret_A}
(A) \leq \frac{1}{2H} \left(\cP_h[V_{h+1}^k - V_{h+1}^{\pi_k}]\right)(\tau_h^k, a_h^k) + \zeta_h^k + 2 c_{\theta} H \sqrt{\frac{(H \wedge D) \iota}{N_{t_h}^k(s_{t_h}^k, a_{t_h}^k, \delta_{t_h}^k)}}.
\end{align}
Substituting \eqref{eq:delay_regret_A} back into \eqref{eq:delay_regret_decomp}, we have
\begin{align*}
\left(V^k_h - V^{\pi_k}_h\right)(\tau_h^k) & = \left(Q^k_h - Q^{\pi_k}_h\right)(\tau_h^k, a_h^k) \\
& \leq \left(1 + \frac{1}{2H}\right) \left(\cP_h\left[V^k_h - V^{\pi_k}_h\right]\right)(\tau_{h}^k, a_h^k) + \zeta_h^k + 2b_h^k + 2 c_{\theta} H \sqrt{\frac{(H \wedge D) \iota}{N_{t_h}^k(s_{t_h}^k, a_{t_h}^k, \delta_{t_h}^k)}}.
\end{align*}
We further denote $\xi_h^k = \left(\cP_h\left[V^k_h - V^{\pi_k}_h\right]\right)(\tau_{h}^k, a_h^k) - \left[V^k_{h+1} - V^{\pi_k}_{h+1}\right](\tau_{h+1}^k)$ and rewrite $\left(V^k_h - V^{\pi_k}_h\right)(\tau_h^k)$ as
\begin{align*}
\left(V^k_h - V^{\pi_k}_h\right)(\tau_h^k) \leq \left(1 + \frac{1}{2H}\right) \left(\left[V^k_{h+1} - V^{\pi_k}_{h+1}\right](\tau_{h+1}^k) + \xi_h^k\right) + \zeta_h^k + 2b_h^k + 2 c_{\theta} H \sqrt{\frac{(H \wedge D) \iota}{N_{t_h}^k(s_{t_h}^k, a_{t_h}^k, \delta_{t_h}^k)}}.
\end{align*}
Recall $\tilde{H} = 2H$. Using a recursive summation argument, we deduce
\begin{align*}
\left(V^k_1 - V^{\pi_k}_1\right)(\tau_1^k) & \leq \sum_{h=1}^{\tilde{H}} \left(1 + \frac{1}{2H}\right)^h \left(\xi_h^k + \zeta_h^k + 2b_h^k + 2 c_{\theta} H \sqrt{\frac{(H \wedge D) \iota}{N_{t_h}^k(s_{t_h}^k, a_{t_h}^k, \delta_{t_h}^k)}}\right) \\
& \leq e \sum_{h=1}^{2H} \left(\xi_h^k + \zeta_h^k + 2b_h^k + 2 c_{\theta} H \sqrt{\frac{(H \wedge D) \iota}{N_{t_h}^k(s_{t_h}^k, a_{t_h}^k, \delta_{t_h}^k)}}\right).
\end{align*}
As a consequence, the total regret is bounded by
\begin{align}\label{eq:regret_sum}
\subopt(K) \leq e \sum_{k=1}^K \sum_{h=1}^{2H} \left(\xi_h^k + \zeta_h^k + 2b_h^k + 2 c_{\theta} H \sqrt{\frac{(H \wedge D) \iota}{N_{t_h}^k(s_{t_h}^k, a_{t_h}^k, \delta_{t_h}^k)}}\right).
\end{align}
We need to sum over $\zeta_h^k, \xi_h^k, b_h^k$. Consider $\zeta_h^k$ first. We have
\begin{align}\label{eq:zeta_sum}
\sum_{k=1}^K \sum_{h=1}^{2H} \zeta_h^k & = c' \sum_{k=1}^K \sum_{h=1}^{2H} \frac{SH^2 \iota}{N_{t_h}^k(s_{t_h}^k, a_{t_h}^k)} \nonumber \\
& \overset{(i)}{\leq} c' (H \wedge D) \sum_{k=1}^K \sum_{h=1}^{H} \frac{SH^2 \iota}{N_{h}^k(s_{h}^k, a_{h}^k)} \nonumber \\
& \overset{(ii)}{\leq} c_{\zeta} (H \wedge D) H^3 S^2 A \iota^2,
\end{align}
where inequality $(i)$ invokes the fact that $t_h$ only takes value in $\{1, \dots, H\}$ and each $N_{t_h}^k(s_{t_h}^k, a_{t_h}^k)$ is repeated at most $(H \wedge D)$ times due to delay, and inequality $(ii)$ follows from the pigeon-hole argument in \cite{azar2017minimax}.

Next we bound the summation over $\xi_h^k$. This is a martingale difference sequence. We apply Azuma-Hoeffding's inequality (Lemma~\ref{lemma:azuma}) with $n = 2H$ and $c_i = 4H$ to obtain
\begin{align}\label{eq:xi_sum}
\sum_{k=1}^K \sum_{h=1}^{2H} \xi_h^k \leq c_{\xi} \sqrt{K (H \wedge D) H^3 \iota}.
\end{align}
The additional $(H \wedge D)$ dependence above comes from a union bound over $\tSaug \times \cA$. Lastly, we tackle the summation over bonus functions $b_h^k$, which consists of summation over two sets of counting numbers $N_{t_h}^k(s, a)$ and $N_{t_h}^k(s, a, \delta)$. For $N_{t_h}^k(s, a)$, we have
\begin{align}\label{eq:bonus_sum}
\sum_{k=1}^K \sum_{h=1}^{2H} b_h^k & = \sum_{k=1}^K \sum_{h=1}^{2H} c_{A} H \sqrt{\frac{(H \wedge D)\iota}{N_{t_h}^k(s_{t_h}, a_{t_h})}} \nonumber \\
& \leq c_A (H \wedge D) \sum_{k=1}^K \sum_{h=1}^H H \sqrt{\frac{(H \wedge D)\iota}{N_{t_h}^k(s_{t_h}, a_{t_h})}} \nonumber \\
& \leq c_b (H \wedge D)^{3/2} \sqrt{H^3 S A K \iota},
\end{align}
where the last inequality follows from Equation (3) in \cite{azar2017minimax}. Putting \eqref{eq:zeta_sum}, \eqref{eq:xi_sum} and \eqref{eq:bonus_sum} together and replacing $(H \wedge D)$ by $H$ in \eqref{eq:zeta_sum} and \eqref{eq:xi_sum}, we deduce
\begin{align*}
\subopt(K) \leq c \left((H \wedge D)^{3/2} \sqrt{H^3 SAK \iota} + H^4 S^2 A \iota^2 + \sqrt{H^4 K \iota} \right) + 2 e c_{\theta} H \sum_{k=1}^K \sum_{h=1}^{2H} \sqrt{\frac{(H \wedge D) \iota}{N_{t_h}^k(s_{t_h}^k, a_{t_h}^k, \delta_{t_h}^k)}}
\end{align*}
for some constant $c$. To this end, the only remaining task is to find $\sum_{k=1}^K \sum_{h=1}^{2H} \sqrt{\frac{1}{N_{t_h}^k(s_{t_h}^k, a_{t_h}^k, \delta_{t_h}^k)}}$, which undergoes a similar argument as the bonus summation. We have
\begin{align}\label{eq:theta_sum}
\sum_{k=1}^K \sum_{h=1}^{2H} \sqrt{\frac{1}{N_{t_h}^k(s_{t_h}^k, a_{t_h}^k, \delta_{t_h}^k)}} & \leq (H \wedge D) \sum_{k=1}^K \sum_{h=1}^{H} \sqrt{\frac{1}{N_{h}^k(s_{h}^k, a_{h}^k, \delta_{h}^k)}} \nonumber \\
& = (H \wedge D) \sum_{(h, s, a, \delta)} \sum_{i=1}^{N_h^K(s, a, \delta)} \sqrt{\frac{1}{i}} \nonumber \\
& \overset{(i)}{\leq} 2 (H \wedge D)\sum_{\delta} \sum_{(h, s, a)} \sqrt{N_h^K(s, a, \delta)} \nonumber \\
& \overset{(ii)}{\leq} 2 (H \wedge D) \sum_{\delta} \sqrt{SAKH} \nonumber \\
& \overset{(iii)}{\leq} 2(H \wedge D)^2 \sqrt{SAKH},
\end{align}
where inequality $(i)$ invokes $\sum_{i=1}^n 1/\sqrt{i} \leq 2\sqrt{n}$, inequality $(ii)$ follows from the Cauchy-Schwarz inequality, and inequality $(iii)$ uses the fact that $\delta$ is bounded by $(H \wedge D)$. Substituting \eqref{eq:theta_sum} into the regret bound, we obtain the desired result
\begin{align*}
\subopt(K) & \leq c \left((H \wedge D)^{5/2} \sqrt{H^3 SAK \iota} + (H \wedge D)^{3/2} H^2 \sqrt{SAK \iota} + H^4 S^2 A \iota^2 + \sqrt{H^4 K \iota} \right) \\
& \leq 2c \left((H \wedge D)^{5/2} \sqrt{H^3 SAK \iota} + H^4 S^2 A \iota^2 + \sqrt{H^4 K \iota} \right)
\end{align*}
with probability $1 - \gamma$. Absorbing $\sqrt{H^4K\iota}$ into $H^4\sqrt{SAK\iota}$ yields the bound in Theorem~\ref{thm:delay_regret}.
\end{proof}

\subsection{Proof of Corollary~\ref{cor:discounted_regret}}\label{pf:discounted_regret}
\begin{proof}
We first show that it suffices to truncate the infinite horizon at $H_{\rm trunc} = \frac{\iota}{1-u}$ for policy learning. We denote $V_\infty^\pi(s) = \EE^{\pi}\left[ \sum_{h=1}^\infty u^h r_h(s_h, a_h) | s\right]$ as the infinite-horizon value function, and $V_{H_{\rm trunc}}^\pi = \EE^{\pi}\left[ \sum_{h=1}^{H_{\rm trunc}} u^h r_h(s_h, a_h) | s\right]$ as the truncated horizon value function. The difference between $V_{\infty}^\pi$ and $V_{H_{\rm trunc}}^\pi$ is bounded by
\begin{align*}
\sup_s~ V_{\infty}^\pi(s) - V_{H_{\rm trunc}}^\pi(s) \leq \sum_{h=H_{\rm trunc}+1}^\infty u^h = u^{H_{\rm trunc}} \frac{1}{1-u},
\end{align*}
where the inequality follows from reward function being bounded. Setting $u^{H_{\rm trunc}} \frac{1}{1-u} \leq \sqrt{\frac{1}{HSAK}}$ leads to $H_{\rm trunc} \geq \left\lceil \frac{\iota}{-\log u} \right\rceil$. Note that when $u \in (0, 1)$, we have $-\log u \geq 1- u$. Therefore, we take $H_{\rm trunc} = \left\lceil \frac{\iota}{1-u} \right\rceil$ so that
\begin{align}\label{eq:horizon_trunc_error}
\sup_s~ V_{\infty}^\pi(s) - V_{H_{\rm trunc}}^\pi(s) \leq \sqrt{\frac{1}{SAKH}}.
\end{align}
The remaining step is to apply Theorem~\ref{thm:delay_regret} for learning in a finite-horizon MDP with $H = H_{\rm trunc}$. The regret is bounded by $c\left((H_{\rm trunc} \wedge D)^{5/2} \sqrt{H_{\rm trunc}^3 SAK\iota} + H_{\rm trunc}^4 S^2 A \iota^2 \right)$ for a sufficiently large constant $c$. Observe that accumulating \eqref{eq:horizon_trunc_error} for $K$ episodes contributes an $\sqrt{\frac{K}{SAH}}$ value difference. Therefore, substituting in our choice of $H_{\rm trunc}$ gives rise to the desired bound.
\end{proof}

\subsection{Proof of Proposition~\ref{prop:pidelay_gap}}\label{pf:pidelay_gap}
\begin{proof}
Let $\tau_1, \dots, \tau_H$ denote the states observed in the delayed environment. Since $\pinodelay$ is greedy and Markov, we obtain
\begin{align*}
\vnodelay{1}(s_1) & = \EE^{\pinodelay} \left[\sum_{h=1}^{H-1} r_h(s_h, a_h) | s_1 \right] + \EE^{\pinodelay} \left[\EE[r_H(s_H, a_H) | \tau_H] | s_1 \right] \\
& = \EE^{\pinodelay} \left[\sum_{h=1}^{H-1} r_h(s_h, a_h) | s_1 \right] + \EE^{\pinodelay} \left[\sum_{s} \fkb_H(s | \tau_H) \max_a r_H(s, a) | s_1 \right].
\end{align*}
Recursively applying the above argument, we deduce
\begin{align*}
\vnodelay{1}(s_1) = \EE^{\pinodelay} \left[\sum_{h=1}^H \sum_s \fkb_h(s | \tau_h) \max_a r_h(s, a) | s_1 \right].
\end{align*}
We also rewrite $\vdelay{1}(s_1)$ as
\begin{align*}
\vdelay{1}(s_1) & = \EE^{\pidelay} \left[\sum_{h=1}^{H-1} r_h(s_h, a_h) | s_1 \right] + \EE^{\pidelay} \left[ \EE[r_H(s_H, a_H) | \tau_H ] | s_1 \right] \\
& = \EE^{\pidelay} \left[\sum_{h=1}^{H-1} r_h(s_h, a_h) | s_1 \right] + \EE^{\pidelay} \left[ \max_a \sum_{s} \fkb_H(s | \tau_H) r_H(s, a) | s_1 \right] \\
& = ... \\
& = \EE^{\pidelay} \left[\sum_{h=1}^H \max_a \sum_s \fkb_h(s | \tau_h) r_h(s, a) | s_1 \right].
\end{align*}
Then we write the difference between $\vnodelay{1}(s_1)$ and $\vdelay{1}(s_1)$ as
\begin{align*}
& \quad \vnodelay{1}(s_1) - \vdelay{1}(s_1) \\
& = \sum_{h=1}^H \Bigg(\int_{\tau} \sum_s \max_a \fkb_h(s | \tau) r_h(s, a) \rho_{h}^{\pinodelay}(\tau) \diff \tau - \int_\tau \max_a \sum_s \fkb_h(s | \tau) r_h(s, a) \rho_{h}^{\pidelay}(\tau) \diff \tau \Bigg) \\
& = \sum_{h=1}^H \Bigg(\int_{\tau} \sum_s \max_a \fkb_h(s | \tau) r_h(s, a) \rho_{h}^{\pinodelay}(\tau) \diff \tau - \int_\tau \max_a \sum_s \fkb_h(s | \tau) r_h(s, a) \rho_{h}^{\pinodelay}(\tau) \diff \tau \\
& \quad + \int_\tau \max_a \sum_s \fkb_h(s | \tau) r_h(s, a) \rho_{h}^{\pinodelay}(\tau) \diff \tau - \int_\tau \max_a \sum_s \fkb_h(s | \tau) r_h(s, a) \rho_{h}^{\pidelay}(\tau) \diff \tau \Bigg) \\
& \leq \sum_{h=1}^H \left[\int_{\tau} \left(\EE_{s \sim \fkb_h(\cdot | \tau)}[\max_a r_h(s, a)] - \max_a \EE_{s \sim \fkb_h(\cdot | \tau)}[r_h(s, a)]\right)\rho_{h}^{\pinodelay}(\tau) \diff \tau + 2 \tvnorm{\rho_{h}^{\pinodelay} - \rho_{h}^{\pidelay}}\right].
\end{align*}
We also have
\begin{align*}
& \quad \vnodelay{1}(s_1) - \vdelay{1}(s_1) \\
& = \sum_{h=1}^H \Bigg(\int_{\tau} \sum_s \max_a \fkb_h(s | \tau) r_h(s, a) \rho_{h}^{\pinodelay}(\tau) \diff \tau - \int_\tau  \sum_s \max_a \fkb_h(s | \tau) r_h(s, a) \rho_{h}^{\pidelay}(\tau) \diff \tau \\
& \quad + \int_\tau \sum_s \max_a \fkb_h(s | \tau) r_h(s, a) \rho_{h}^{\pidelay}(\tau) \diff \tau - \int_\tau \max_a \sum_s \fkb_h(s | \tau) r_h(s, a) \rho_{h}^{\pidelay}(\tau) \diff \tau \Bigg) \\
& \leq \sum_{h=1}^H \left[\int_{\tau} \left(\EE_{s \sim \fkb_h(\cdot | \tau)}[\max_a r_h(s, a)] - \max_a \EE_{s \sim \fkb_h(\cdot | \tau)}[r_h(s, a)]\right)\rho_{h}^{\pidelay}(\tau) \diff \tau + 2 \tvnorm{\rho_{h}^{\pinodelay} - \rho_{h}^{\pidelay}}\right].
\end{align*}
Combining the above two inequalities, we obtain
\begin{align*}
& \quad \vnodelay{1}(s_1) - \vdelay{1}(s_1) \\
& \leq \sum_{h=1}^H \Bigg[\int_{\tau} \left(\EE_{s \sim \fkb_h(\cdot | \tau)}[\max_a r_h(s, a)] - \max_a \EE_{s \sim \fkb_h(\cdot | \tau)}[r_h(s, a)]\right)\left(\rho_{h}^{\pidelay} \wedge \rho_h^{\pinodelay} \right)(\tau) \diff \tau \\
& \quad + 2 \tvnorm{\rho_{h}^{\pinodelay} - \rho_{h}^{\pidelay}}\Bigg].
\end{align*}
The proof is complete.
\end{proof}

\subsection{Proof of Proposition~\ref{prop:d+1_vs_d}}\label{pf:d+1_vs_d}
\begin{proof}
We construct an MDP instance $(\cS, \cA, H, R, P)$ for $H > d$ as follows. Let $\cS = \{1, 2\}$ and $\cA = \{a_1, a_2\}$. For the reward function, we have
\begin{align*}
r_h(s, a) = 
\begin{cases}
1 & \text{if}~a = a_s~\text{and}~ h = d+1 \\
0 & \text{otherwise}
\end{cases}.
\end{align*}
The reward is nonzero only at time $d+1$. The transition probabilities are defined as
\begin{align*}
p_h(s' | s, a) =
\begin{cases}
\frac{1}{2} & \text{if}~h = d+1 \\
1 & \text{if}~ h \neq d+1~\text{and}~ s' = s \\
0 & \text{otherwise}
\end{cases}.
\end{align*}
The transition probability at step $d+1$ says that $s'$ is uniform regardless of the previous state and action. Assume a uniform initial distribution on $s_1$. We first show that if the constant delay equals $d$, then there exists a policy $\pi^{*, d}$ achieving the maximal value of reward. Indeed, the policy is chosen as
\begin{align*}
\pi_h^{*, d}(\cdot | \{s_{h-d}, \as{h-d}{h-1}\}) =
\begin{cases}
a_{s_{h-d}} & \text{if}~h = d+1 \\
\text{Uniform}(\cA) & \text{if}~h \neq d+1.
\end{cases}
\end{align*}
It is straightforward to check that $\pi^{*, d}$ is optimal, since at step $d+1$, $s_1$ is revealed and the policy takes the optimal action $a_{s_1}$ to obtain reward $1$.

On the other hand, if the constant delay equals $d+1$, then any policy suffers from a constant performance degradation. To see this, in a single trajectory, since the starting state is only revealed at time $d+2$, the policy at time $d+1$ cannot exploit the information of the initial state. Therefore, any policy coincides with the best action with probability $\frac{1}{2}$.
For $K$ episodes, with probability $1 - \gamma$, the total reward of any policy $\pi \in \pilo$ is bounded by
\begin{align*}
\sum_{k=1}^K V_1^{\pi}(s_1^k) \leq \frac{1}{2}K + \sqrt{\frac{K}{2} \log \frac{1}{\gamma}},
\end{align*}
due to Hoeffding's inequality. As a result, the performance drop is at least by
\begin{align*}
\gap(K) \geq \frac{1}{2} - \sqrt{\frac{1}{2K} \log \frac{1}{\gamma}}.
\end{align*}
\end{proof}

\subsection{Proof of Proposition~\ref{prop:regret_lowerbound}}\label{pf:regret_lowerbound}
\begin{proof}
We prove the lower bound by constructing a hard MDP instance adapted from \cite{domingues2021episodic}. 

\paragraph{Hard instance construction} Let $S_w$ be a waiting state and $S_b$ and $S_g$ be two absorbing states. The remaining states are arranged in a tree structure. The absorbing states $S_b$ and $S_g$ are directly reachable from the leaves of the tree. See Figure~\ref{fig:hardinstance} for an illustration.
\begin{figure}[!htb]
\centering
\includegraphics[width = 0.7\textwidth]{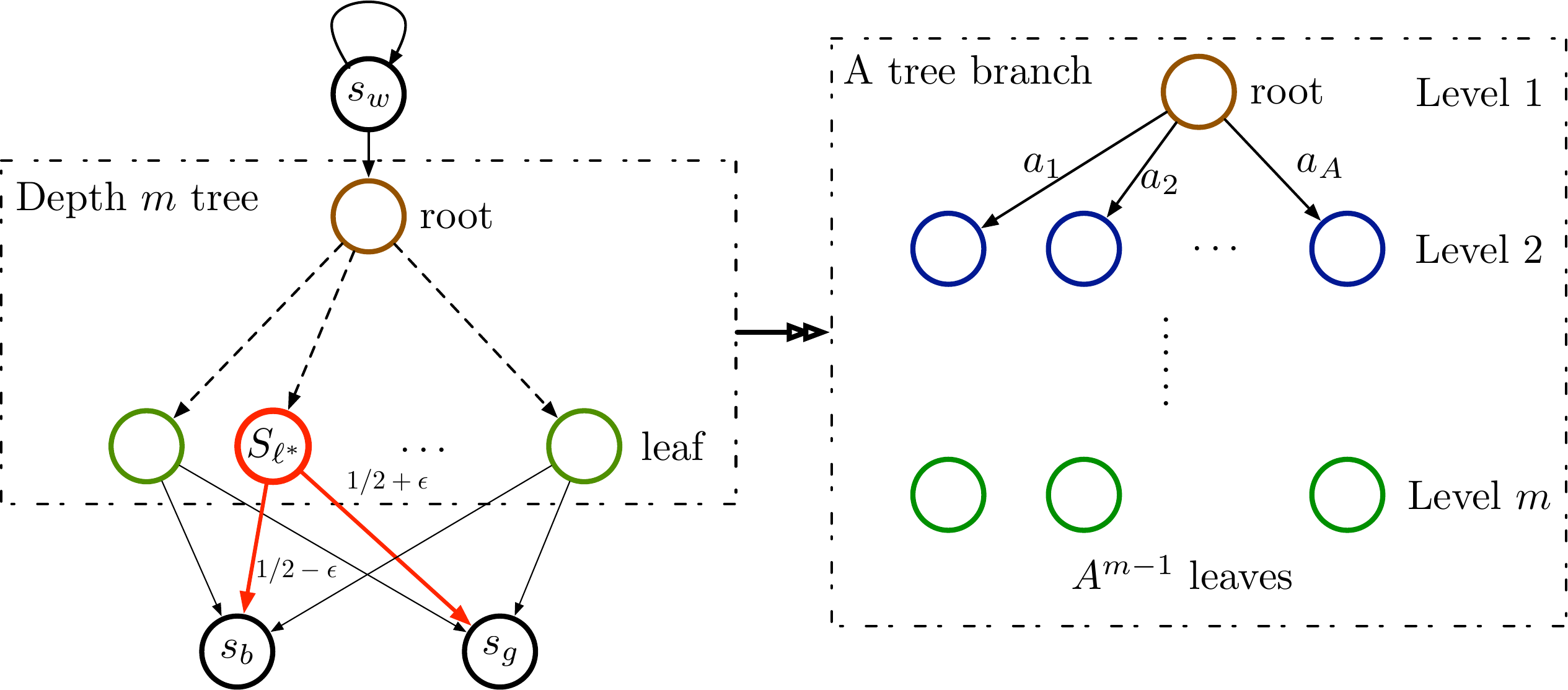}
\caption{Illustration of the constructed hard MDP instance in the left panel. The right panel is a zoomed in illustration of a saturated tree branch.}
\label{fig:hardinstance}
\end{figure}

In the sequel, we assume that the depth of the tree is $m$ and $S = 3 + \frac{A^m - 1}{A-1}$, i.e., the depth $m$ tree is saturated. We further assume that the horizon $H \geq 3m$. These assumptions simplify the presentation and can be removed as shown in Appendix D of \cite{domingues2021episodic}. We define the reward to be zero on all states except $S_g$; on $S_g$, it holds that
\begin{align*}
r_h(s_g, a) = \mathds{1}\{h \ge D + m + 1\} \quad \text{for~any}~ a \in \cA.
\end{align*}
The transition probability at $S_w$ follows
\begin{align*}
\PP_h(S_w | S_w, a) = \mathds{1}\{a = a_w, h \leq D\} \quad \text{and} \quad \PP_h(S_{\rm root} | S_w, a) = 1 - \PP_h(S_w | S_w, a).
\end{align*}
In words, $S_w$ only transits to $S_w$ and $S_{\rm root}$. The special action $a_w$ allows the state to stay at $S_w$ for $D$ steps. After that, one must transits from $S_w$ to $S_{\rm root}$.

In the tree, the transition is deterministic and one can only move downwards in the tree. That is, traveling from the root to a leaf takes $m$ steps. We use $\ell$ to index the leaf states. Assume $(h^*, \ell^*, a^*)$ is a triple of special time, leaf and action, where $h^* \in \{1+m, D+m\}$. Then the transition from a leaf state $S_\ell$ to $S_b$ and $S_g$ is totally random except at leaf $\ell^*$ with action $a^*$, i.e., $$\PP_h(S_b | S_\ell, a) = 1/2 - \epsilon \cdot \mathds{1}\{h = h^*, \ell = \ell^*, a = a^*\} \quad \text{and} \quad \PP_h(S_g | S_\ell, a) = 1 - \PP_h(S_b | S_\ell, a),$$
where $\epsilon \leq 1/4$ is a small probability to be determined later. The transition to $S_b$ and $S_g$ says that at some time $h^*$, there is a special leaf $S_{\ell*}$ with action $a^*$ that enables slightly increased probability to reach the ``good'' state $S_g$. Therefore, finding the optimal policy is equivalent to identifying $(h^*, \ell^*, a^*)$. We shall take $(h^*, \ell^*, a^*)$ uniformly random.

\paragraph{Lower bound analysis} Leveraging the insights from Proposition~\ref{prop:d+1_vs_d}, the optimal policy with $D$-step delay is the same as no delay, since the transition in the constructed instance is deterministic except for $S_b$ and $S_g$. To this end, we can apply the analysis in Theorem 9 of \cite{domingues2021episodic} with the substitution of their $\bar{H}, d, T$ to $D$, $m$ and $K$, respectively. Following Equation (12) in \cite{domingues2021episodic} and taking $\epsilon = \frac{1}{2\sqrt{2}} (1- 1/(DLA)) \sqrt{\frac{D L A}{K}} \leq 1/4$, we obtain
\begin{align*}
\EE[\subopt(K)] \geq \frac{1}{4\sqrt{2}} \left(1 - \frac{1}{D L A}\right) (H - D - m)\sqrt{D L A K},
\end{align*}
where $L = A^{m-1}$ is the number of leaves and can be rewritten as $L = (S-3) (1-1/A) + 1/A \geq S/4$. By the assumption of $H \geq 3m$ and $H \geq 2D$, we deduce
\begin{align*}
\EE[\subopt(K)] \geq \frac{1}{72\sqrt{2}} H \sqrt{DSAK}.
\end{align*}
\end{proof}

\section{Omitted proofs in Section~\ref{sec:regret_missing}}

\subsection{Proof of Proposition~\ref{prop:s2arate}}\label{pf:s2arate}
\begin{proof}
We first show that the ground-truth transition probabilities $p^{\theta^*}_h$ belongs to $\cB_k$ with high probability. By Theorem 2.2 in \cite{weissman2003inequalities} (see also Equation (44) in \cite{jaksch2010near}), at the $k$-th episode, for any fixed $(s, a, h)$, we have
\begin{align*}
\PP\left(\norm{\hat{p}_h^k(\cdot | s, a) - p^{\theta^*}_h(\cdot | s, a)}_{\rm TV} \geq t \right) \leq (2^S - 2) \exp\left(- \frac{N_h^k(s, a) t^2}{2}\right).
\end{align*}
Setting $t = c\sqrt{\frac{S\iota}{N_h^k(s, a)}}$ for some constant $c$ ensures that
\begin{align*}
\norm{\hat{p}_h^k(\cdot | s, a) - p^{\theta^*}_h(\cdot | s, a)}_{\rm TV} \leq c\sqrt{\frac{S\iota}{N_h^k(s, a)}}
\end{align*}
holds over any $(s, a, h, k)$ with probability $1 - \gamma$. As a consequence, the event $p^{\theta^*}_h(\cdot | s, a) \in \cB^k$ holds with probability $1- \gamma$ over all $(s, a, h, k)$.

Conditioned on the high probability event $p^{\theta^*}_h \in \cB^k$ for all $(h, s, a)$, we have by standard performance difference arguments that
\begin{align*}
\sum_{k=1}^K \max_{\pi\in\pilo} V_{\theta^\star}^\pi(s_1^k) - V_{\theta^\star}^{\pi^k}(s_1^k) & \overset{(i)}{\le} \sum_{k=1}^K V_{\theta^k}^{\pi^k}(s_1^k) - V_{\theta^\star}^{\pi^k}(s_1^k) \\
& \overset{(ii)}{=} \sum_{k=1}^K \sum_{h=1}^H \E_{\theta^\star}^{\pi^k}\brac{ \< (\P_h^{\theta^k} - \P_h^{\theta^\star})(\cdot|s_h,a_h), V^{\pi^k}_{\theta^k,h+1}(\cdot)  \> } \\
& \le \sum_{h=1}^H \sum_{k=1}^K \E_{\theta^\star}^{\pi^k}\brac{ 
c\sqrt{\frac{H^2S\iota}{N_h^k(s_h,a_h)}} \wedge H } \\
& \overset{(iii)}{\le} \sum_{h=1}^H \sum_{k=1}^K c'\sqrt{\frac{H^2S\iota}{N_h^k(s_h^k,a_h^k)}} + H\sqrt{H^2K\iota} \\
& \overset{(iv)}{\le} c' \paren{\left\lceil \frac{\log \frac{HK}{\gamma}}{-\log (1-\lambda_0^2)} \right\rceil\sqrt{H^2S\iota \cdot SAHK} + \sqrt{H^4K\iota}} \\
& \leq c' \left(\left\lceil \frac{1}{-\log (1-\lambda_0^2)} \right\rceil\sqrt{H^3 S^2 A K \iota^3} + \sqrt{H^4 K \iota}\right),
\end{align*}
where inequality $(i)$ follows from the valid optimism since Line~\ref{alg:greedy} in Algorithm~\ref{alg:opt_planning} is taken over double maximization, equality $(ii)$ recursively expands the value function and $\inner{\cdot}{\cdot}$ denotes the inner product, inequality $(iii)$ invokes Azuma-Hoeffding's inequality, and inequality $(iv)$ invokes Lemma~\ref{lemma:round_robin_sum}.
\end{proof}

\subsection{Proof of Theorem~\ref{thm:regret_missing}}\label{pf:regret_missing}
\begin{proof}
The proof utilizes similar steps as Theorem~\ref{thm:delay_regret}, with extra care on the summation of bonus functions.

\paragraph{Valid optimism} We verify that the choice of bonus functions leads to a valid optimism.
\begin{lemma}\label{lemma:optimism_missing}
Given any failure probability $\gamma < 1$, we set the bonus functions as
\begin{align*}
b_h^k(\tau_h, a_h) = c H \left(\sqrt{\frac{H\iota}{N_h^k(\tau_h, a_h)}} + \sqrt{\frac{\iota}{k}}\right) \quad \text{with} \quad \iota = \log \left(\frac{SAHK}{\gamma}\right).
\end{align*}
Then with probability $1 - \gamma$, it holds that
\begin{align*}
\Qaug{h}^k(\tau_h, a_h) \geq \Qaug{h}^*(\tau_h, a_h), \quad \Vaug{h}^k(\tau_h) \geq \Vaug{h}^*(\tau_h) \quad \text{for~any}\quad (k, h, \tau_h, a_h).
\end{align*}
\end{lemma}
\begin{proof}[Proof of Lemma~\ref{lemma:optimism_missing}]
In the proof, we omit subscripts ``aug'' for simplicity. We use backward induction on time $h$ again. The base case of $H+1$ holds immediately due to the initial value of $V_{H+1}$. Suppose the assertion holds at time $h+1$. Then for time $h$, if $Q_h = H$, the assertion holds trivially. Otherwise, we have
\begin{align*}
Q_h(\tau_h, a_h) - Q_h^*(\tau_h, a_h) & = \hat{r}_h(\tau_h, a_h) + [\hat{\cP}_hV_{h+1}](\tau_h, a_h) - r_h(\tau_h, a_h) - [\cP_h V_{h+1}^*](\tau_h, a_h) + b_h^k(\tau_h, a_h) \\
& \geq \underbrace{\left([\hat{\cP}_h - \cP_h]V_{h+1}^*\right)(\tau_h, a_h)}_{(A)} + \underbrace{\hat{r}_h(\tau_h, a_h) - r_h(\tau_h, a_h)}_{(B)} + ~b_h^k(\tau_h, a_h).
\end{align*}
We lower bound terms $(A)$ and $(B)$ separately. For term $(A)$, we have
\begin{align*}
(A) & = \sum_{\tau_{h+1}} V_{h+1}^*(\tau_{h+1}) \left(\hat{p}_h(\tau_{h+1} | \tau_h, a_h) - p_h(\tau_{h+1} | \tau_h, a_h) \right) \\
& = \sum_{\tau_{h+1}} V_{h+1}^*(\tau_{h+1}) \left(\hat{p}_h(\tau_{h+1} | \tau_h, a_h) - p_h(\tau_{h+1} | \tau_h, a_h) \right) \mathds{1}\{t_{h+1} = h+1\} \\
& \quad + \sum_{\tau_{h+1}} V_{h+1}^*(\tau_{h+1}) \left(\hat{p}_h(\tau_{h+1} | \tau_h, a_h) - p_h(\tau_{h+1} | \tau_h, a_h) \right) \mathds{1}\{t_{h+1} = t_h\} \\
& = \underbrace{\sum_{s_{h+1}} V_{h+1}^*(\tau_{h+1}) \left((1-\hat{\lambda}_h) \hat{p}_h(s_{h+1} | s_{t_h}, \as{t_h}{h}) - (1-\lambda_h) p_h(s_{h+1} | s_{t_h}, \as{t_h}{h}) \right)}_{(A_1)} \\
& \quad + \underbrace{V_{h+1}^*(\{s_{t_h}, \as{t_h}{h}\}) (\hat{\lambda}_h - \lambda_h)}_{(A_2)}.
\end{align*}
In $(A_1)$, $\tau_{h+1}$ is the pure state $\{s_{h+1}\}$. We bound $(A_1)$ as
\begin{align*}
(A_1) & = \sum_{s_{h+1}} V_{h+1}^*(\tau_{h+1}) \Big((1-\hat{\lambda}_h) \hat{p}_h(s_{h+1} | s_{t_h}, \as{t_h}{h}) - (1-\lambda_h) \hat{p}_h(s_{h+1} | s_{t_h}, \as{t_h}{h}) \\
& \quad + (1-\lambda_h) \hat{p}_h(s_{h+1} | s_{t_h}, \as{t_h}{h}) - (1-\lambda_h) p_h(s_{h+1} | s_{t_h}, \as{t_h}{h}) \Big) \\
& = \sum_{s_{h+1}} V_{h+1}^*(\tau_{h+1}) (1-\lambda_h) \left(\hat{p}_h(s_{h+1} | s_{t_h}, \as{t_h}{h}) - p_h(s_{h+1} | s_{t_h}, \as{t_h}{h})\right) \\
& \quad + \sum_{s_{h+1}} V_{h+1}^*(\tau_{h+1}) (\lambda_h -\hat{\lambda}_h) \hat{p}_h(s_{h+1} | s_{t_h}, \as{t_h}{h}) \\
& \overset{(i)}{\geq} - c_{A} H \sqrt{\frac{H \iota}{N_{h}(\tau_h, a_h)}} - H \left| \hat{\lambda}_h - \lambda_h \right|,
\end{align*}
where inequality $(i)$ invokes Hoeffding's inequality and holds with probability $1 - \gamma$ for any $(\tau_h, a_h, h, k)$ and some constant $c_{A}$. Term $(A_2)$ is immediately bounded by
\begin{align*}
(A_2) \geq - H \left| \hat{\lambda}_h - \lambda_h \right|.
\end{align*}
Putting $(A_1)$ and $(A_2)$ together, we derive
\begin{align*}
(A) \geq - c_{A} H \sqrt{\frac{H \iota}{N_{h}(\tau_h, a_h)}} - 2 H \left| \hat{\lambda}_h - \lambda_h \right|
\end{align*}
with high probabilty. For term $(B)$, we have
\begin{align*}
(B) = \sum_{s_h} r(s_h, a_h) \left(\hat{\fkb}_h(s_h | \tau_h) - \fkb_h(s_h | \tau_h)\right) \geq - c_B \sqrt{\frac{H\iota}{N_h(\tau_h, a_h)}}.
\end{align*}
Taking $c = c_A + c_B$ and summing up $(A)$ and $(B)$, we have
\begin{align*}
Q_h(\tau_h, a_h) - Q_h^*(\tau_h, a_h) \geq - c H \sqrt{\frac{H\iota}{N_h(\tau_h, a_h)}} - 2H \left|\hat{\lambda}_h - \lambda_h \right| + b_h^k(\tau_h, a_h).
\end{align*}
We estimate $\lambda_h$ by its empirical average. In episode $k \geq 1$, we have access to $k$ i.i.d. realizations of a Bernoulli random variable with rate $\lambda_h$ (observable or not). Therefore, by Hoeffding's inequality, we have
\begin{align*}
\left| \hat{\lambda}^k_h - \lambda_h \right| \leq 2 \sqrt{\frac{\log \frac{HK}{\gamma}}{k}} \leq 2 \sqrt{\frac{\iota}{k}}.
\end{align*}
Substituting into $Q_h^k(\tau_h, a_h) - Q_h^*(\tau_h, a_h)$ and reloading constant $c$ sufficiently large give rise to
\begin{align*}
Q_h^k(\tau_h, a_h) - Q_h^*(\tau_h, a_h) \geq - c H \left( \sqrt{\frac{H\iota}{N_h^k(\tau_h, a_h)}} + \sqrt{\frac{\iota}{k}} \right) + b_h^k(\tau_h, a_h) \geq 0.
\end{align*}
The proof is complete.
\end{proof}

\paragraph{Regret analysis} We omit subscripts ``aug'' to ease the presentation. The same derivation in the proof of Theorem~\ref{thm:delay_regret} gives rise to
\begin{align}\label{eq:regret_missing_decomp}
& \quad \left(Q_h^* - Q_h^{\pi_k}\right)(\tau_h^k, a_h^k) \leq \left(Q^k_h - Q^{\pi_k}_h\right)(\tau_h^k, a_h^k) \nonumber \\
& \leq \underbrace{\left([\hat{\cP}_h^k - \cP_h][V^k_{h+1} - V_{h+1}^*]\right)(\tau_h^k, a_h^k)}_{(A)} + \left(\cP_h[V_{h+1}^k - V_{h+1}^{\pi_k}]\right)(\tau_h^k, a_h^k) + 2b_h^k(\tau_{h}^k, a_{h}^k).
\end{align}
Lemma~\ref{lemma:optimism_missing} shows that $(A)$ can be written as
\begin{align*}
(A) & = \sum_{s_{h+1}} [V^k_{h+1} - V_{h+1}^*](\tau_{h+1}) (1-\lambda_h) \left(\hat{p}_h^k(s_{h+1} | s_{t_h}^k, \as{t_h}{h}^k) - p_h(s_{h+1} | s_{t_h}^k, \as{t_h}{h}^k)\right) \\
& \quad + \sum_{s_{h+1}} [V^k_{h+1} - V_{h+1}^*](\tau_{h+1}) (\lambda_h -\hat{\lambda}^k_h) \hat{p}_h^k(s_{h+1} | s_{t_h}^k, \as{t_h}{h}^k) \\
& \leq \sum_{s_{h+1}} [V^k_{h+1} - V_{h+1}^*](\tau_{h+1}) (1-\lambda_h) \left(\hat{p}_h^k(s_{h+1} | s_{t_h}^k, \as{t_h}{h}^k) - p_h(s_{h+1} | s_{t_h}^k, \as{t_h}{h}^k)\right) + H \left|\hat{\lambda}^k_h - \lambda_h \right| \\
& \leq (1-\lambda_h) \sum_{s_{h+1}} [V^k_{h+1} - V_{h+1}^*](\tau_{h+1}) \left(\hat{p}_h^k(s_{h+1} | s_{t_h}^k, \as{t_h}{h}^k) - p_h(s_{h+1} | s_{t_h}^k, \as{t_h}{h}^k)\right) + 2 H \sqrt{\frac{\iota}{k}}.
\end{align*}
Following the derivation in \eqref{eq:A_bernstein}, \eqref{eq:delay_regret_A} and \eqref{eq:regret_sum}, we have
\begin{align*}
\subopt(K) & \leq e \sum_{k=1}^K \sum_{h=1}^{H} \left(\xi_h^k + \zeta_h^k + 2b_h^k + 2 H \sqrt{\frac{\iota}{k}}\right) \\
& \leq e \sum_{k=1}^K \sum_{h=1}^{H} \left(\xi_h^k + \zeta_h^k + 2b_h^k \right) + 2 \sqrt{H^4 K \iota}.
\end{align*}
where $\xi_h^k = \left(\cP_h\left[V^k_h - V^{\pi_k}_h\right]\right)(\tau_{h}^k, a_h^k) - \left[V^k_{h+1} - V^{\pi_k}_{h+1}\right](\tau_{h+1}^k)$ is a martingale difference and $\zeta_{h}^k = c' \frac{SH^2 \iota}{N_{h}^k(\tau_h^k, a_h^k)}$.

\paragraph{Summation of counting numbers} The summation over $\xi_h^k$ is standard. Using Azuma-Hoeffding's inequality, we have
\begin{align*}
\sum_{k=1}^K \sum_{h=1}^{H} \xi_h^k \leq c_\xi \sqrt{KH^4 \iota}.
\end{align*}
It remains to find the summations involving $N_h^k(\tau_h^k, a_h^k)$. First, we show that the event $\cE_m = \{h - t_h - 1 \leq m\}$, i.e., the maximal consecutive delay is upper bounded by $m > 0$, holds with high probability. We have
\begin{align*}
\PP(\cE_m) \leq \left(1 - H(1-\lambda_0)^{m+1}\right)^K,
\end{align*}
since $\lambda_0$ is a uniform lower bound on $\lambda_h$. Next, we provide an upper bound on $N_h^K(\tau_h, a_h)$. For a given tuple $(h, \tau_h, a_h, t_h)$, the consecutive missing length is $h - t_h - 1$. Such a missing pattern appears with probability at most $(1-\lambda_0)^{h-t_h -1}$. As a consequence, denote $C_{h-t_h-1}^K$ as the number of $h-t_h-1$ consecutive missings in $K$ episodes. With probability $1 - \gamma$, we have
\begin{align*}
C_{h-t_h-1}^K \leq K (1-\lambda_0)^{h-t_h-1} + \sqrt{K (1-\lambda_0)^{h - t_h - 1} H \iota} + \iota.
\end{align*}
by Bernstein's inequality in Lemma~\ref{lemma:bernstein}. Furthermore, at a fixed time $h$, we use Lemma~\ref{lemma:max_gap} to bound the gap between two consecutive appearances of the same missing pattern. We instantiate Lemma~\ref{lemma:max_gap} with $\theta = (1-\lambda_0)^{h - t_h - 1}$ and obtain that the gap is bounded by $\left\lceil \frac{\iota}{-\log (1-(1-\lambda_0)^{h-t_h-1})} \right\rceil$ with probability $1 - \gamma$. Within the gap, the number of consecutive delays of length larger than $h - t_h - 1$ is bounded by
\begin{align*}
C_{\geq h-t_h-1} & \overset{(i)}{\leq} \left\lceil \frac{\iota}{-\log (1-(1-\lambda_0)^{h-t_h-1})} \right\rceil (1-\lambda_0)^{h-t_h} \\
& \quad + \sqrt{\left\lceil \frac{\iota}{-\log (1-(1-\lambda_0)^{h-t_h-1})} \right\rceil (1-\lambda_0)^{h - t_h} H \iota} + \iota \\
& \overset{(ii)}{\leq} \sqrt{2(1-\lambda_0) H \iota} + 2(1-\lambda_0) + \iota,
\end{align*}
where inequality $(i)$ follows from Bernstein's inequality again and inequality $(ii)$ invokes the fact $x + \log (1-x) \leq 0$ for $x \in [0, 1)$ and bounds $\lceil x \rceil$ by $x + 1$. Now we can bound the summation of the counting numbers. Conditioned on the event $\cE_m$, we have
\begin{align*}
\sum_{k=1}^K \sum_{h=1}^{H} \sqrt{\frac{1}{N_{h}^k(\tau_{h}^k, a_{h}^k)}} & \overset{(i)}{\leq} \sum_{(h, \tau, a, t_h)} C_{\geq h-t_h-1} \sum_{i=1}^{N_h^K(\tau, a)} \sqrt{\frac{1}{i}} \nonumber \\
& \leq 2\left(\sqrt{2(1-\lambda_0) H \iota} + 2(1-\lambda_0) + \iota\right) \sum_{(h, \tau, a, t_h)} \sqrt{N_h^K(\tau, a)} \nonumber \\
& \overset{(ii)}{\leq} 2 \left(\sqrt{2(1-\lambda_0) H \iota} + 2(1-\lambda_0) + \iota\right) \sum_{h, t_h} \sqrt{SA^{h-t_h} C_{h-t_h-1}^K} \nonumber \\
& \leq 2 \left(\sqrt{2(1-\lambda_0) H \iota} + 2(1-\lambda_0) + \iota\right) \\
& \quad \cdot \sum_{h, t_h} \sqrt{SA \left(K((1-\lambda_0) A)^{h-t_h-1} + \sqrt{K (A^2(1-\lambda_0))^{h-t_h-1} H\iota} + A^{h-t_h-1} \iota\right)}\nonumber \\
& \overset{(iii)}{\leq} 2 \left(\sqrt{2(1-\lambda_0) H \iota} + 2(1-\lambda_0) + \iota\right) \sum_{h, t_h} \sqrt{SA \left(K + \sqrt{K A^m H\iota} + A^m \iota\right)}\nonumber \\
& \leq 2 \left(\sqrt{2(1-\lambda_0) H \iota} + 2(1-\lambda_0) + \iota\right) H^2 \sqrt{SA \left(K + \sqrt{K A^m H\iota} + A^m \iota\right)} \\
& \leq 2 \sqrt{H^5SA \iota^2 \left(K + \sqrt{KA^mH \iota} + A^m \iota\right)},
\end{align*}
where inequality $(i)$ follows since $N_h^k$ is repeated at most $C_{\geq h-t_h-1}$ times before getting an update and inequality $(ii)$ follows from the Cauchy-Schwarz inequality, and inequality $(iii)$ invokes the assumption $\lambda A \leq 1$. Moreover, conditioned on the event $\cE_m$, we also have
\begin{align*}
\sum_{k=1}^K \sum_{h=1}^{H} \frac{1}{N_{h}^k(\tau_{h}^k, a_{h}^k)} & \leq \sum_{(h, \tau, a, t_h)} C_{\geq h-t_h-1} \sum_{i=1}^{N_h^K(\tau, a)} \frac{1}{i} \nonumber \\
& \leq \left(\sqrt{2(1-\lambda_0) H \iota} + 2(1-\lambda_0) + \iota\right) \sum_{(h, \tau, a, t_h)} \log N_h^K(\tau, a) \nonumber \\
& \leq \iota H^{5/2} SA^{m+1} \log K. \nonumber
\end{align*}

\paragraph{Combining the above} On event $\cE_m$, the regret is bounded by
\begin{align*}
\subopt(K) & \overset{(i)}{\leq} c \left(\sqrt{H^4 K \iota} + \sum_{k=1}^K \sum_{h=1}^H \left[\frac{SH^2 \iota}{N_h^k(\tau_h^k, a_h^k)} + H \sqrt{\frac{H\iota}{N_h^k(\tau_h^k, a_h^k)}} \right]\right) \\
& \leq c \left(H^4 \sqrt{SA \iota^3 K \left(1 + \sqrt{\frac{A^m H \iota}{K}} + \frac{A^m \iota}{K} \right)} + S^2 A^m \sqrt{H^9 \iota^6} + \sqrt{H^4 K \iota} \right),
\end{align*}
where $c$ is a sufficiently large constant and we substitute the bonus functions into inequality $(i)$.

On the complement of $\cE_m$, the regret is bounded by $H(1 - \PP(\cE_m)) \leq H^2 K (1-\lambda_0)^{m+1}$. We choose $m = \frac{1}{2} \left\lfloor \frac{\log K}{-\log (1-\lambda_0)} \right\rfloor$ such that $H(1 - \PP(\cE_m)) \leq H^2 K (1-\lambda_0)^{m+1} \leq H^2 \sqrt{K}$. We can now check that $A^{m+1} = \exp\left(\frac{\log A}{-\log (1-\lambda_0)} \log \sqrt{K} \right) \leq K^{\frac{1}{2(1+v)}}$. Therefore, combining the regret on event $\cE_m$ and the complement event $\cE_m^{\complement}$ leads to
\begin{align*}
\subopt(K) \leq c \left(H^4 \sqrt{SA K\iota^3} + S^2 \sqrt{H^9 K^{\frac{1}{(1+v)}} \iota^6} \right).
\end{align*}
The proof is complete.
\end{proof}

\subsection{Supporting Lemmas}
\begin{lemma}\label{lemma:round_robin_sum}
Suppose Assumption~\ref{assumption:missing} holds. With probability $1 - \gamma$ for some failure probability $\gamma > 0$, we have
\begin{align*}
& \sum_{k=1}^K \sum_{h=1}^H \frac{1}{\sqrt{N_h^k(s_h^k, a_h^k)}} \leq \left\lceil \frac{\log \frac{HK}{\gamma}}{-\log (1-\lambda_0^2)} \right\rceil \sqrt{SAKH}.
\end{align*}
\end{lemma}
\begin{proof}[Proof of Lemma~\ref{lemma:round_robin_sum}]
For any time $h$, we denote $\cK^{\rm eff}(h)$ as the collection of episodes that the $h$-th and $(h+1)$-th step observations are available. It is clear that the cardinality of $\cK^{\rm eff}(h)$ is bounded by $K$ for any $h$. Within each $\cK^{\rm eff}(h)$, we would like to bound the gap between two observations. Thanks to Lemma~\ref{lemma:max_gap}, the gap is bounded by $q$ with probability $1 - K (1-\lambda_0^2)^{q+1}$. We set $K (1-\lambda_0^2)^{q+1} = \gamma / H$, which implies $q = \left\lceil \frac{\log \frac{HK}{\gamma}}{-\log (1-\lambda_0^2)} \right\rceil$. Therefore, for any time step $h$, available observations are at most separated by $q$ episodes.

With this notation, we bound
\begin{align*}
\sum_{k=1}^K \sum_{h=1}^H \frac{1}{\sqrt{N_h^k(s_h^k, a_h^k)}} & \overset{(i)}{\le} \left\lceil \frac{\log \frac{HK}{\gamma}}{-\log (1-\lambda_0^2)} \right\rceil \sum_{h=1}^{H} \sum_{k \in \cK^{\rm eff}(h)} \frac{1}{\sqrt{N_h^k(s_h^k, a_h^k)}} \\
& \overset{(ii)}{\leq} \left\lceil \frac{\log \frac{HK}{\gamma}}{-\log (1-\lambda_0^2)} \right\rceil \sum_{h=1}^{H} \sum_{k=1}^{K} \frac{1}{\sqrt{N_h^k(s_h^k, a_h^k)}} \\
& \overset{(iii)}{\leq} 2 \left\lceil \frac{\log \frac{HK}{\gamma}}{-\log (1-\lambda_0^2)} \right\rceil \sqrt{SAHK}, 
\end{align*}
where inequality $(i)$ follows since $N_h^k$ will only be updated when $h \in \cK^{\rm eff}(h)$ and then repeat at most $\left\lceil \frac{\log \frac{HK}{\gamma}}{-\log (1-\lambda_0^2)} \right\rceil$ times, inequality $(ii)$ invokes the cardinality bound on $\cK^{\rm eff}(h)$, and inequality $(iii)$ follows from the standard pigeon-hole principle.
\end{proof}

\begin{lemma}\label{lemma:max_gap}
Let $\{u_i\}_{i=1}^k$ be i.i.d. Bernoulli random variables. Suppose $\PP(u_i = 1) = \theta$. Define the largest gap between $u_i$'s as
\begin{align*}
g(k) = \sup \{j - i : u_i = 0 \text{~and~} u_j = 0 \text{~with~} u_\ell = 1 \text{~for~} \ell = i+1, \dots, j-1\}.
\end{align*}
Then for any integer $q > 0$, the following tail probability bound holds,
\begin{align*}
\PP(g(k) > q) \leq k \theta^{q+1}.
\end{align*}
\end{lemma}
\begin{proof}[Proof of Lemma~\ref{lemma:max_gap}]
We denote $I_{\rm neg} = \{\ell_1, \dots, \ell_m\}$ as the index set for $u_{\ell_i} = 0$ when $i = 1, \dots, |I_{\rm neg}|$. Let $v_j = \ell_{j+1} - \ell_{j}$, which is a geometric random variable with a success rate $\theta$. Note that the cardinality of $I_{\rm neg}$ is at most $k$. Therefore, we have
\begin{align*}
\PP(g(k) > q) & \leq \PP(\max_{j = 1, \dots, k} v_j > q) \\
& = 1 - \PP\left(v_j \leq q ~\text{for}~ j = 1, \dots, k \right) \\
& = 1 - \left(1 - \theta^{q+1}\right)^k \\
& \leq k \theta^{q+1},
\end{align*}
where the last inequality follows from $ 1 - k \theta^{q+1} \leq (1 - \theta^{q+1})^k$.
\end{proof}

\section{Helper concentration inequalities}
\begin{lemma}[Bernstein's inequality]\label{lemma:bernstein}
Let $x_1, \dots, x_n$ be i.i.d. zero mean random variables. Suppose $|x_i| \leq M$ for any $i = 1, \dots, n$. Then for all positive $t$, it holds that
\begin{align*}
\PP\left(\sum_{i=1}^n x_i > t\right) \leq \exp\left(-\frac{\frac{1}{2}t^2}{\sum_{i=1}^n \Var[x_i] + \frac{1}{3} M t}\right).
\end{align*}
In particular, given a failure probability $\gamma < 1$, it holds that
\begin{align*}
\PP\left(\sum_{i=1}^n x_i > \sqrt{\sum_{i=1}^n \Var[x_i] \log \frac{1}{\gamma}} + M \log \frac{1}{\gamma}\right) \leq \gamma.
\end{align*}
\end{lemma}
\begin{proof}[Proof of Lemma~\ref{lemma:bernstein}]
The proof of Bernstein's inequality is standard, see for example \cite[Section 2.1]{wainwright2019high}. Here we verify the second claim. Let $\exp\left(-\frac{\frac{1}{2}t^2}{\sum_{i=1}^n \Var[x_i] + \frac{1}{3} M t}\right) \leq \gamma$ hold true. We find a suitable $t$ by
\begin{align*}
& \exp\left(-\frac{\frac{1}{2}t^2}{\sum_{i=1}^n \Var[x_i] + \frac{1}{3} M t}\right) \leq \gamma \\
\Longleftrightarrow~ & \frac{\frac{1}{2}t^2}{\sum_{i=1}^n \Var[x_i] + \frac{1}{3} M t} \geq \log \frac{1}{\gamma} \\
\Longleftrightarrow~ & t^2 - \frac{2}{3} t M \log \frac{1}{\gamma} \geq \sum_{i=1}^n \Var[x_i] \log \frac{1}{\gamma} \\
\Longleftrightarrow~ & t \geq \sqrt{\sum_{i=1}^n \Var[x_i] \log \frac{1}{\gamma} + \frac{1}{9} M^2 \log^2 \frac{1}{\gamma}} + \frac{1}{3} M \log \frac{1}{\gamma}.
\end{align*}
It is enough to choose $t = \sqrt{\sum_{i=1}^n \Var[x_i] \log \frac{1}{\gamma}} + M \log \frac{1}{\gamma}$.
\end{proof}

\begin{lemma}[Hoeffding's inequality]\label{lemma:hoeffding}
Let $x_1, \dots, x_n$ be i.i.d. random variables. Suppose $a_i \leq x_i \leq b_i$ for any $i = 1, \dots, n$. Then for all positive $t$, it holds that
\begin{align*}
\PP\left(\left| \sum_{i=1}^n x_i - \EE\left[\sum_{i=1}^n x_i \right] \right| > t\right) \leq 2 \exp\left(-\frac{2t^2}{\sum_{i=1}^n (b_i - a_i)^2}\right).
\end{align*}
In particular, given a failure probability $\gamma < 1$, it holds that
\begin{align*}
\PP\left(\frac{1}{n} \left| \sum_{i=1}^n x_i - \EE\left[\sum_{i=1}^n x_i \right] \right| > \sqrt{\frac{\sum_{i=1}^n (b_i - a_i)^2 \log \frac{2}{\gamma}}{2n^2}} \right) \leq \gamma.
\end{align*}
\end{lemma}
\begin{proof}[Proof of Lemma~\ref{lemma:hoeffding}]
The proof is standard; see \cite[Section 2.1]{wainwright2019high}.
\end{proof}

\begin{lemma}[Azuma-Hoeffding's inequality]\label{lemma:azuma}
Let $x_1, \dots, x_n$ be a martingale adapted to a filtration $\cF_1 \subset \dots \subset \cF_n$. Suppose $\EE[x_i - \EE[x_i] | \cF_{i-1}] = 0$ and $|x_i - \EE[x_i]| \leq c_i$. Then for all positive $t$, it holds that
\begin{align*}
\PP\left(\sum_{i=1}^n x_i - \EE[x_i] > t\right) \leq \exp\left(-\frac{t^2}{2 \sum_{i=1}^n c_i^2}\right).
\end{align*}
In particular, given a failure probability $\gamma < 1$, it holds that
\begin{align*}
\PP\left(\sum_{i=1}^n x_i - \EE[x_i] > \sqrt{2 \sum_{i=1}^n c_i^2 \log \frac{1}{\gamma}} \right) \leq \gamma.
\end{align*}
\end{lemma}
\begin{proof}[Proof of Lemma~\ref{lemma:azuma}]
The proof is standard and applies Lemma~\ref{lemma:hoeffding}.
\end{proof}

\end{document}